\documentclass[tablecaption=bottom,wcp]{jmlr}

\usepackage[T1]{fontenc}

\usepackage{mathtools}
\usepackage{thmtools}
\usepackage{enumitem}
\newlist{inlinelist}{enumerate*}{1}
  \setlist*[inlinelist,1]{%
          label=(\roman*),
      }
\usepackage{bbm}
\usepackage{booktabs}
\usepackage[load-configurations=version-1]{siunitx}
\usepackage{cleveref}
\usepackage{tikz}

\usepackage{caption}
\theorembodyfont{\upshape}
\theoremheaderfont{\bfseries\scshape}
\theorempostheader{:}
\newtheorem{prop}{Prop}[section]

\newtheorem{dfn}{Def}[section]

\newtheorem{lem}{Lem}[section]
\newtheorem{cor}{Cor}[section]
\newtheorem{thm}{Thm}[section]

\crefname{prop}{Prop}{Props}
\crefname{fact}{Fact}{Facts}
\crefname{dfn}{Def}{Defs}
\crefname{ass}{Asm}{Asms}
\crefname{lem}{Lem}{Lems}
\crefname{cor}{Cor}{Cors}
\crefname{thm}{Thm}{Thms}
\crefname{appendix}{App}{Apps}
\crefname{section}{Sec}{Secs}
\crefname{table}{Tab}{Tabs}
\crefname{figure}{Fig}{Figs}
\crefname{equation}{}{}

\Crefname{prop}{Prop}{Props}
\Crefname{fact}{Fact}{Facts}
\Crefname{dfn}{Def}{Defs}
\Crefname{ass}{Asm}{Asms}
\Crefname{lem}{Lem}{Lems}
\Crefname{cor}{Cor}{Cors}
\Crefname{thm}{Thm}{Thms}
\Crefname{appendix}{App}{Apps}
\Crefname{section}{Sec}{Secs}
\Crefname{table}{Tab}{Tabs}
\Crefname{figure}{Fig}{Figs}
\Crefname{equation}{}{}

\newlist{asslist}{enumerate}{1}
\setlist[asslist]{label=(\roman{asslisti}),
                  ref=\theass.(\roman{asslisti})}
\crefname{listass}{Asm}{Asms}
\Crefname{listass}{Asm}{Asms}
\addtotheorempostheadhook[ass]{\crefalias{asslisti}{listass}}

\newlist{proplist}{enumerate}{1}
\setlist[proplist]{label=(\roman{proplisti}),
                  ref=\theprop.(\roman{proplisti})}
\crefname{listprop}{Prop}{Props}
\Crefname{listprop}{Prop}{Props}
\addtotheorempostheadhook[prop]{\crefalias{proplisti}{listprop}}

\newlist{lemlist}{enumerate}{1}
\setlist[lemlist]{label=(\roman{lemlisti}),
                  ref=\thelem.(\roman{lemlisti})}
\crefname{listlem}{Lem}{Lems}
\Crefname{listlem}{Lem}{Lems}
\addtotheorempostheadhook[lem]{\crefalias{lemlisti}{listlem}}

\usepackage{bm}
\usepackage{upgreek}

\newcommand{\vnull}{\mathbf{0}}
\newcommand{\ve}{\mathbf{e}}
\newcommand{\vn}{\mathbf{n}}
\newcommand{\vm}{\mathbf{m}}

\newcommand{\vx}{\mathbf{x}}
\newcommand{\vy}{\mathbf{y}}
\newcommand{\vz}{\mathbf{z}}

\newcommand{\vtau}{\mathbf{\uptau}}
\newcommand{\vth}{\mathbf{\uptheta}}
\newcommand{\vpsi}{\mathbf{\uppsi}}

\newcommand{\mA}{\mathbf{A}}
\newcommand{\mB}{\mathbf{B}}
\newcommand{\mH}{\mathbf{H}}
\newcommand{\mI}{\mathbf{I}}
\newcommand{\mK}{\mathbf{K}}

\newcommand{\mZ}{\mathbf{Z}}

\newcommand{\cf}{\textit{c.f.}}
\newcommand{\ie}{\textit{i.e.}}
\newcommand{\eg}{\textit{e.g.}}

\let\set\setold
\input{math}

\renewcommand{\G}{\operatorname{G}}

\newcommand{\spacer}{\vspace*{0.2em}}
\newcommand{\sectioncompressed}[1]{\vspace*{-5pt}\section{#1}}

\jmlrproceedings{AABI 2020}{3rd Symposium on Advances in Approximate Bayesian Inference, 2020}

\title[
    The Gaussian Neural Process
]{
    The Gaussian Neural Process
}

\author{
    \Name{Wessel P.\ Bruinsma}
    \hfill
    \makebox[7cm][r]{\addr University of Cambridge and Invenia Labs}
    \makebox[3.5cm][r]{\texttt{wpb23@cam.ac.uk}}
    \\
    \Name{James Requeima}
    \hfill
    \makebox[7cm][r]{\addr University of Cambridge and Invenia Labs}
    \makebox[3.5cm][r]{\texttt{jrr41@cam.ac.uk}}
    \\
    \Name{Andrew Y.\ K.\ Foong}
    \hfill
    \makebox[7cm][r]{\addr University of Cambridge}
    \makebox[3.5cm][r]{\texttt{ykf21@cam.ac.uk}}
    \\
    \Name{Jonathan Gordon}
    \hfill
    \makebox[7cm][r]{\addr University of Cambridge}
    \makebox[3.5cm][r]{\texttt{jg801@cam.ac.uk}}
    \\
    \Name{Richard E.\ Turner}
    \hfill
    \makebox[7cm][r]{\addr University of Cambridge}
    \makebox[3.5cm][r]{\texttt{ret26@cam.ac.uk}}
}

\begin{document}

\maketitle

\begin{abstract}
Neural Processes \citep[NPs;][]{Garnelo:2018:Conditional_Neural_Processes,Garnelo:2018:Neural_Processes} are a rich class of models for meta-learning that map data sets directly to predictive stochastic processes.
We provide a rigorous analysis of the standard maximum-likelihood objective used to train conditional NPs.
Moreover, we propose a new member to the Neural Process family called the Gaussian Neural Process (\textsc{GNP}),
which models predictive correlations,
incorporates translation equivariance,
provides universal approximation guarantees, and
demonstrates encouraging performance.
\end{abstract}%

\section{Introduction}
Neural Processes \citep[NPs;][]{Garnelo:2018:Conditional_Neural_Processes,Garnelo:2018:Neural_Processes} use neural networks to directly parameterise and learn a map from observed data to posterior predictive distributions of a stochastic process.
In this work, we provide two contributions to the NP framework.

Our first contribution is a rigorous analysis of the standard maximum-likelihood (ML) objective used to train conditional NP models.
In particular, we relate the objective to the KL divergence between stochastic processes \citep{Matthews:2016:On_Sparse_Variational}, which we call a \emph{functional} KL.
For a ground truth $\P$ and approximating process $\Q$, learning procedures that minimise a functional KL have previously been investigated \citep{Sun:2018:Functional_Variational_Bayesian_Neural_Networks,Shi:2019:Scalable_Training_of_Inference_Networks,Ma:2018:Variational_Implicit_Processes},
but these works leave important questions about finiteness of the objective and existence/uniqueness of its minimiser unanswered.
In this work, we consider the objective $\KL(\P, \Q)$.
In a well-defined and rigorous setup, we demonstrate that the ML objective can be interpreted as a well-behaved relaxation of this functional objective.

Our second contribution addresses the inability of conditional NPs \citep[CNPs;][]{Garnelo:2018:Conditional_Neural_Processes} to model correlations and produce coherent samples.
Several authors propose to overcome this limitation by introducing a latent variable \citep{Garnelo:2018:Neural_Processes,Kim:2019:Attentive_Neural_Processes,Foong:2020:Meta-Learning_Stationary_Stochastic_Process_Prediction}.
Unfortunately, this renders the likelihood intractable, complicating learning and evaluation.
Building on the \textsc{ConvCNP} \citep{Gordon:2020:Convolutional_Conditional_Neural_Processes},
we introduce the \emph{Gaussian} NP (\textsc{GNP}), a novel member of the NP family that incorporates \emph{translation equivariance} and models the predictive distributions directly with Gaussian processes \citep[GPs;][]{Rasmussen:2006:Gaussian_Processes}.
The \textsc{GNP} allows for correlations in the predictive distribution whilst admitting a closed-form likelihood.
Moreover, like the \textsc{ConvCNP}, the \textsc{GNP} provides universal approximation guarantees, which we showcase by providing empirical evidence that the \textsc{GNP} can recover the \emph{prediction map} of a ground-truth Gaussian process in terms of likelihood and prior covariance function.

\sectioncompressed{A Practical Objective for Meta-Learning with Gaussian Processes}
\label{sec:objective-function-analysis}

A detailed description of the notation and terminology used in this section can be found in \cref{app:notation}.
The statements and proofs of all theorems are deferred to \cref{app:Gaussian_divergence,app:noisy_processes,app:objective}.

\textbf{\scshape Problem setup:}
Let $f$ be a ground-truth stochastic process.
In the meta-learning setup,
we aim to make multiple predictions for $f$ based on a collection of observed data sets $(D_i)_{i=1}^N$ drawn from $f$.
With access to $f$, these predictions are given by the posteriors over $f$ given $(D_i)_{i=1}^N$.
We can view prediction as a map from observed data sets $\D$ to posteriors over $f$.
This map is called the \emph{posterior prediction map} $\pi_f\colon \mathcal{D} \to \mathcal{P}$ \citep{Foong:2020:Meta-Learning_Stationary_Stochastic_Process_Prediction}.
Our goal is to learn a \emph{Gaussian} approximation $\tilde \pi \colon \D \to \mathcal{P}\ss{G}$ of $\pi_f$ (\cref{def:prediction_map}) that approximates the posteriors over $f$ with Gaussian processes.
Note that a Gaussian approximation of the posterior prediction map is distinctly different from learning a Gaussian approximation of the prior $f$:
the only requirement on $\tilde \pi$ is that $\tilde \pi(D)$ is a Gaussian process for all $D \in \D$;
in particular, these GPs are not constrained to be posteriors obtained from a fixed prior, which means that learning $\tilde \pi$ enjoys significantly more flexibility. 
In fact, this setup is strictly more flexible than the originally proposed \textsc{CNP} \citep{Garnelo:2018:Conditional_Neural_Processes}, as the \textsc{CNP} can be viewed as a map $\D \to \mathcal{P}\ss{G}$ that does not model correlations.

\spacer
\textbf{\scshape Functional objective:}
We directly define our approximation $\tilde \pi$ of $\pi_f$:
for every $D \in \D$, approximate $\pi_f(D)$ with a Gaussian process $\mu$:
\begin{equation} \label{eq:approximation_of_process}
    \textstyle
    \tilde \pi(D) = \argmin_{\mu \in \mathcal{P}\ss{G}}\, \KL(\pi_f(D), \mu).
\end{equation}
Under reasonable regularity conditions and the assumption that
there exists some non-degenerate Gaussian process $\mu\ss{G} \in \mathcal{P}\ss{G}$ such that $\KL(\pi_f(D), \mu\ss{G}) < \infty$, this minimiser exists and is unique (\cref{cor:existence_uniqueness}).
However, such a Gaussian process $\mu\ss{G}$ may not exist.
Moreover, even if the minimiser $\tilde \pi(D)$ exists and is unique, meaning that \eqref{eq:approximation_of_process} is finite at $\tilde \pi(D)$, there may not exist a ball of approximations around $\tilde \pi(D)$ for which the objective \eqref{eq:approximation_of_process} is finite;
in that case, the minimiser $\tilde \pi(D)$ cannot be approximated by minimising \eqref{eq:approximation_of_process}.
For example, suppose that $\pi_f(D) = \GP(0, k)$ where $k(t, t') = \exp(-\tfrac12(t-t')^2)$.
Set $\mu_{\sigma^2} = \GP(0, \sigma^2 k)$.
Then a quick computation shows that $\KL(\pi_f(D), \mu_{\sigma^2}) = \infty$ for all $\sigma^2 \neq 1$.
Hence, we cannot recover the true variance $\sigma^2 = 1$ by initialising $\mu_{\sigma^2}$ with some reasonable $\sigma^2 > 0$ and minimising $\KL(\pi_f(D), \mu_{\sigma^2})$, because the objective is infinite for all but the true value of $\sigma^2$.

\spacer
\textbf{\scshape Relaxation:}
To work around the potential absence of a minimiser, we take a pragmatic stance and instead simply approximate the \emph{finite-dimensional distributions} (f.d.d.s): \sloppy
\begin{equation} \label{eq:approximation_of_fdds}
    \textstyle
    \tilde \pi^\vx(D) = \argmin_{\mu^\vx \in \mathcal{P}^{\abs{\vx}}\ss{G}}\, \KL(P_\vx \pi_f(D), \mu^\vx)
    \quad \text{for all finite index sets $\vx$,}
\end{equation}
where $P_\vx f = (f(x_1),\ldots, f(x_{\abs{\vx}}))$ for $f \in \Y^\X$ is the projection onto the index set $\vx$.
Under reasonable regularity conditions and the assumption that,
for all finite index sets $\vx$, there exists an appropriate $\abs{\vx}$-dimensional Gaussian distribution $\mu^{\vx}\ss{G}$ such that $\KL(P_\vx \pi_f(D), \mu^{\vx}\ss{G}) < \infty$, these minimisers exist and are unique (\cref{prop:consistent_fdds}).
This condition is much milder than that for \eqref{eq:approximation_of_process}:
it is satisfied for any appropriate $\mu^{\vx}\ss{G}$ if the differential entropy of $P_\vx \pi_f(D)$ is finite.
Crucially, it turns out that \eqref{eq:approximation_of_fdds} gives rise to a consistent collection of f.d.d.s (\cref{prop:consistent_fdds}) and therefore uniquely defines an approximating process $\tilde \pi(D)$ satisfying $P_\vx \tilde \pi(D) = \tilde \pi^\vx(D)$ for all finite index sets $\vx$.
Moreover, if a solution to \eqref{eq:approximation_of_process} exists, then it will be equal to $\tilde \pi(D)$ (\cref{prop:consistent_fdds,cor:existence_uniqueness}).
Therefore, \eqref{eq:approximation_of_fdds} defines a relaxation of \eqref{eq:approximation_of_process} that can be used in many cases where a solution to \eqref{eq:approximation_of_process} does not exist.
The solution to \eqref{eq:approximation_of_process} and \eqref{eq:approximation_of_fdds}, if it exists, is given by the \emph{moment-matched Gaussian process}:
the Gaussian process obtained by taking the mean function and covariance function of $\pi_f(D)$; see also \citet{Ma:2018:Variational_Implicit_Processes}.

\spacer
\textbf{\scshape Approximable objective:}
The workaround \eqref{eq:approximation_of_fdds} solves the problem of \emph{existence}.
However, there is still the problem of \emph{approximability}:
\eqref{eq:approximation_of_process} cannot always be minimised to approximate the minimiser, if one exists.
We therefore define another objective, one that is always finite and consequently can always be minimised to approximate the solution to \eqref{eq:approximation_of_fdds}.
This objective is obtained by averaging \eqref{eq:approximation_of_fdds} over index sets of a fixed size:
\begin{equation} \label{eq:KL_approximation_of_expectation_of_fdds}
    \textstyle
    \tilde\pi(D)
    = \argmin_{\mu \in \overline{\mathcal{P}}\ss{G}}\, \E_{p(\vx)}[\KL(P_\vx \pi_f(D), P_\vx \mu)]
\end{equation}
where $p(\vx)$ is a Borel distribution with full support over all index sets of a fixed size $n \ge 2$.
(See \cref{def:noisy_process} for the definition of $\overline{\mathcal{P}}\ss{G}$.)
This objective is well defined (\cref{prop:measurability_KL}).
If (i) the mean and covariances functions of $\pi_f(D)$ and $\mu$ exist and are uniformly bounded by $M > 0$ and (ii) the processes $\pi_f(D)$ and $\mu$ are noisy (\cref{def:noisy_process}), then \cref{prop:bound} shows that the objective is finite and consequently suitable for optimisation;
\cref{prop:minimiser_measure} shows that the minimisers of \eqref{eq:approximation_of_fdds} and \eqref{eq:KL_approximation_of_expectation_of_fdds} are equal.
A useful feature of \eqref{eq:KL_approximation_of_expectation_of_fdds} is that it averages over index sets of a fixed size $n \ge 2$, unlike previous objectives, \eg~Prop 1 by \citet{Foong:2020:Meta-Learning_Stationary_Stochastic_Process_Prediction}, which requires an average over index sets of all sizes.

\spacer
\textbf{\scshape Practical objective:}
We further average \eqref{eq:KL_approximation_of_expectation_of_fdds} over an appropriate selection of data sets, which formulates a single objective that captures the total approximation error of $\tilde \pi$:
\begin{equation} \label{eq:objective}
    \textstyle
    \tilde \pi
    = \argmin_{\pi \in \overline{\M}\ss{G}\us{f.d.d.}} \, \E_{p(D)p(\vx)}[\KL(P_\vx \pi_f(D), P_\vx \pi(D))]
\end{equation}
where $p(D)$ is a Borel distribution with full support over a collection of data sets $\tilde \D \sub \D$ that is open and bounded (\cref{def:bounded_collection_of_data_sets}).
(See \cref{def:continuous_prediction_map} for the definition of $\overline{\M}\ss{G}\us{f.d.d.}$.)
This objective is also well defined (\cref{prop:measurability_E_KL}).
Under conditions similar to the conditions for \eqref{eq:KL_approximation_of_expectation_of_fdds}, \eqref{eq:objective} is finite
and the minimisers of \eqref{eq:approximation_of_fdds} and \eqref{eq:objective} are equal (\cref{prop:bound_over_data_sets,prop:minimiser_map}).
We therefore propose to learn $\tilde \pi$ though minimising \eqref{eq:objective}.
In practice, we optimise a Monte Carlo approximation of \eqref{eq:objective}.
Let $(D_i)_{i=1}^N \sub \D$ be a collection of data sets, all sampled from $f$ and split up $D_i = (D\us{(c)}_i, D\us{(t)}_i)$ into \emph{context sets} $D_i\us{(c)}$ and \emph{target sets} $D_i\us{(t)} = (\vx\us{(t)}_i, \vy\us{(t)}_i)$ \citep{Vinyals:2016:Matching_Networks_for_One_Shot,Ravi:2017:Optimization_as_a_Model_for}.
We then maximise
\begin{equation} \label{eq:objective_MC}
    \textstyle
    \tilde \pi
    \approx \argmax_{\pi \in \overline{\M}\ss{G}\us{f.d.d.}} \frac{1}{N} \sum_{i=1}^N \log \Normal(\vy\us{(t)}_i\cond \vm_i, \mK_i)
    \;\text{ with }\;
    P_{\vx\us{(t)}_i} \pi(D\us{(c)}_i) = \Normal(\vm_i, \mK_i),
\end{equation}
where we ignore irrelevant additive constants that do not depend on $\pi$.
This objective is exactly the standard maximum likelihood objective that is used to train conditional NP models \citep{Garnelo:2018:Conditional_Neural_Processes,Gordon:2020:Convolutional_Conditional_Neural_Processes}.
Analysis of the minimising procedure is difficult and depends on the details of the particular algorithm.
What we can say, however, is that a minimising sequence either diverges or converges to the right limit; and, under certain conditions, a minimising sequence always has a convergent subsequence (\cref{prop:minimising_sequence}).

\sectioncompressed{The Gaussian Neural Process}
\label{sec:GNP}

Having defined a suitable objective, to learn the approximation $\tilde \pi$ in practice, we proceed to generally parametrise $\tilde \pi$.
In this paper, we confine ourselves to \emph{stationary} ground truths $f$.

\spacer
\textbf{\scshape Translation equivariance:}
\citet{Foong:2020:Meta-Learning_Stationary_Stochastic_Process_Prediction} show
that stationarity of $f$ is equivalent to \emph{translation equivariance} (TE) of the posterior prediction map $\pi_f$ of $f$: for all $D \in \D$,
$
    \T_\tau \pi_f(D) = \pi_f(D + \tau)\text{ for all $\tau \in \X$}
$
where $\T_\tau f = f(\vardot - \tau)$ is the \emph{shifting operator}, $\T_\tau  \pi_f(D)$ is the measure $\pi_f(D)$ pushed through $\T_\tau$, and
$D + \tau = (\vx, \vy) + \tau = ((x_1 + \tau, \ldots, x_{\abs{\vx}} + \tau), \vy)$.
If $\pi_f$ is TE, then it is reasonable to restrict our approximation $\tilde \pi$ to also be TE.
Incorporating translation equivariance directly into the model has been shown to yield large improvements in generalisation capability, parameter efficiency, and predictive performance \citep{Gordon:2020:Convolutional_Conditional_Neural_Processes,Foong:2020:Meta-Learning_Stationary_Stochastic_Process_Prediction}.
Denote $\tilde \pi(D) = \GP(m(D), k(D))$ where $m \colon \D \to C(\X, \Y)$ is the TE \emph{mean map} of our approximation $\tilde \pi$ and $k \colon \D \to C\us{p.s.d.}(\X^2, \Y)$ is the TE \emph{kernel map};
see \cref{app:GNP} for more details.

\spacer
\textbf{\scshape Universal parametrisation of mean and kernel:}
For the mean map $m$, we use a \textsc{ConvDeepSet} architecture \citep[used in the \textsc{ConvCNP},][]{Gordon:2020:Convolutional_Conditional_Neural_Processes}, which can approximate any translation-equivariant map from data sets to continuous mean functions \citep[Thm 1 by][]{Gordon:2020:Convolutional_Conditional_Neural_Processes}.
Unfortunately, as we explain in \cref{subsec:kernel_representation}, the \textsc{ConvDeepSet} architecture is not directly applicable to the kernel map $k$.
In \cref{app:GNP}, we modify the \textsc{ConvDeepSet} architecture to make it suitable for the kernel.
This architecture has universal approximation guarantees similar to the \textsc{ConvDeepSet} (\cref{thm:kernel_representation}) and thus completes a general approximate parametrisation of $\tilde \pi$.
Intuitively, the architecture works as follows.
A covariance function is a function $\X \times \X \to \R$ and can therefore be interpreted as an \emph{image} (\eg, imagine that $\X = \set{1, \ldots, n}$).
Whereas the \textsc{ConvCNP} generates the mean by embedding the data in a 1D array and passing it through 1D convolutions, the architecture for the kernel similarly embeds the data in a 2D \emph{image} and passes it through 2D convolutions.
Let $D\us{(c)}$ be a context set and $\vx\us{(t)}$ inputs of a target set.
Then the covariance matrix $\mK\us{(t)}$ at the target points $\vx\us{(t)}$ is generated as follows:
\newcommand*\circled[1]{\tikz[baseline=(char.base)]{
    \node[shape=circle, draw, inner sep=0.5pt, white, fill=black] (char) {\small \bfseries #1};
}}%
\begin{equation} \label{eq:kernel_architecture}
    \circled{1}\; \mH = \operatorname{enc}(D\us{(c)}, \mZ),
    \quad
    \circled{2}\; \mK = \Pi\us{p.s.d.} \operatorname{CNN}(\mH), \quad
    \circled{3}\; \mK\us{(t)} = \operatorname{dec}(\mK, \vx\us{(t)}):
\end{equation}
\begin{enumerate}
    \item[\circled{1}]
        $\mH = \operatorname{enc}(D\us{(c)}, \mZ)$ maps the target set $D\us{(c)}$ to an \emph{encoding} $\mH \in \R^{M \times M \times 3}$ at a prespecified grid $\mZ \in \R^{M \times M}$ for some $M \in \N$ (\cf~the discretisation in the \textsc{ConvCNP} \citep{Gordon:2020:Convolutional_Conditional_Neural_Processes}),
        comprising a
        \emph{data channel} $\mH_{::1}$ (\cf~the data channel in the \textsc{ConvDeepSet}),
        \emph{density channel} $\mH_{::2}$ (\cf~the density channel in the \textsc{ConvDeepSet}), and
        \emph{source channel} $\mH_{::3} = \mI$ (not present in the \textsc{ConvDeepSet}; see \cref{subsec:kernel_architecture});
    \item[\circled{2}]
        $\mK = \Pi\us{p.s.d.} \operatorname{CNN}(\mH)$ passes the encoding $\mH$ through a CNN, producing an $M \times M$ matrix, and projects this matrix with $\Pi\us{p.s.d.}$ onto the nearest positive semi-definite (PSD) matrix with respect to the Frobenius norm \citep{Higham:1988:Nearest_PSD}; and
    \item[\circled{3}]
        $\mK\us{(t)} = \operatorname{dec}(\mK, \vx\us{(t)})$ finally interpolates the obtained PSD matrix $\mK$ to the desired covariances $\mK\us{(t)}$ for the target inputs $\vx\us{(t)}$.
\end{enumerate}
The architecture and the precise definitions of $\operatorname{enc}$ and $\operatorname{dec}$ are described in more detail in \cref{subsec:kernel_architecture}.
In our experiments, contrary to the description above and \cref{subsec:kernel_architecture}, we substitute \circled{2} with the simpler operation $\mK = \operatorname{CNN}(\mH) \operatorname{CNN}(\mH)^\T$, which also guarantees positive semi-definiteness.
It is unclear whether this substitution limits the expressivity of the resulting architecture or interferes with translation equivariance.
We leave an investigation of the implementation of \circled{2} with $\Pi\us{p.s.d.}$ for future work.

\spacer
\textbf{\scshape Source channel:}
A novel aspect of the architecture \eqref{eq:kernel_architecture} is the \emph{source channel} $\mH_{::3} = \mI$, which is simply the identity matrix and not present in the \textsc{ConvDeepSet} architecture.
Intuitively, the source channel allows the architecture to ``start out'' with a stationary prior with covariance $\mI$, corresponding to white noise, and then pass it through a CNN to modulate this prior to introduce correlations inferred from the context set.
\emph{C.f.}, any Gaussian process can be sampled from by first sampling white noise and then convolving with an appropriate filter;
the kernel architecture is a nonlinear generalisation of this procedure.

\spacer
\textbf{\scshape The Gaussian Neural Process:}
The \textsc{ConvDeepSet} for the mean mapping $m$ and the above architecture for the kernel mapping $k$ form a model that we call the Gaussian Neural Process (\textsc{GNP}).
The \textsc{GNP} depends on some parameters $\vth$, \eg~weights and biases for the CNNs.
To train these parameters $\vth$, we maximise \eqref{eq:objective_MC}.
See \cref{subsec:training_objective} for more details.

\begin{table}[t]
\small
\centering
\begin{tabular}{lccccc}
\toprule
 & \multicolumn{1}{c}{\textsc{EQ}} & \multicolumn{1}{c}{\textsc{Mat\'ern--$\frac52$}} & \multicolumn{1}{c}{\textsc{Weakly Per.}} & \multicolumn{1}{c}{\textsc{Sawtooth}} & \multicolumn{1}{c}{\textsc{Mixture}}\\
\midrule
\textit{GP (truth)} & $0.70 { \scriptstyle \,\pm\, 4.8\text{\textsc{e}}{\,\text{--}3} }$ & $0.31 { \scriptstyle \,\pm\, 4.8\text{\textsc{e}}{\,\text{--}3} }$ & $\text{--}0.32 { \scriptstyle \,\pm\, 4.3\text{\textsc{e}}{\,\text{--}3} }$ & n/a & n/a \\
\textsc{GNP} & $\mathbf{0.70} { \scriptstyle \,\pm\, 5.0\text{\textsc{e}}{\,\text{--}3} }$ & $\mathbf{0.30} { \scriptstyle \,\pm\, 5.0\text{\textsc{e}}{\,\text{--}3} }$ & $\mathbf{\text{\textbf{--}}0.47} { \scriptstyle \,\pm\, 5.0\text{\textsc{e}}{\,\text{--}3} }$ & $0.42 { \scriptstyle \,\pm\, 0.01 }$ & $\mathbf{0.10} { \scriptstyle \,\pm\, 0.02 }$ \\
\textsc{ConvNP} & $\text{--}0.46 { \scriptstyle \,\pm\, 0.01 }$ & $\text{--}0.67 { \scriptstyle \,\pm\, 9.0\text{\textsc{e}}{\,\text{--}3} }$ & $\text{--}1.02 { \scriptstyle \,\pm\, 6.0\text{\textsc{e}}{\,\text{--}3} }$ & $\mathbf{1.20} { \scriptstyle \,\pm\, 7.0\text{\textsc{e}}{\,\text{--}3} }$ & $\text{--}0.50 { \scriptstyle \,\pm\, 0.02 }$ \\
\textsc{ANP} & $\text{--}0.61 { \scriptstyle \,\pm\, 0.01 }$ & $\text{--}0.75 { \scriptstyle \,\pm\, 0.01 }$ & $\text{--}1.19 { \scriptstyle \,\pm\, 5.0\text{\textsc{e}}{\,\text{--}3} }$ & $0.34 { \scriptstyle \,\pm\, 7.0\text{\textsc{e}}{\,\text{--}3} }$ & $\text{--}0.69 { \scriptstyle \,\pm\, 0.02 }$ \\ \midrule
\textit{GP (truth, no corr.)} & $\text{--}0.81 { \scriptstyle \,\pm\, 0.01 }$ & $\text{--}0.93 { \scriptstyle \,\pm\, 0.01 }$ & $\text{--}1.18 { \scriptstyle \,\pm\, 7.0\text{\textsc{e}}{\,\text{--}3} }$ & n/a & n/a \\
\textsc{ConvCNP} & $\text{--}0.80 { \scriptstyle \,\pm\, 0.01 }$ & $\text{--}0.95 { \scriptstyle \,\pm\, 0.01 }$ & $\text{--}1.20 { \scriptstyle \,\pm\, 7.0\text{\textsc{e}}{\,\text{--}3} }$ & $0.55 { \scriptstyle \,\pm\, 0.02 }$ & $\text{--}0.93 { \scriptstyle \,\pm\, 0.02 }$ \\
\bottomrule
\end{tabular}
\caption{
    Likelihoods for the 1D regression experiments for interpolation inside the training range.
    Highlights best performance.
    Only the \textsc{ConvCNP} does not model correlations.
    The errors are 95\%-confidence intervals.
    See \cref{app:experiments}.
}
\label{tab:1D_results}
\end{table}

\begin{figure}[t]
    \small
    \centering
    \makebox[0.19\linewidth][c]{  \small\textsc{EQ}}%
    \hfill%
    \makebox[0.19\linewidth][c]{   \small\textsc{Matern--$\frac52$}}%
    \hfill%
    \makebox[0.19\linewidth][c]{   \small\textsc{Weakly Per.}}%
    \hfill%
    \makebox[0.19\linewidth][c]{   \small\textsc{Sawtooth}}%
    \hfill%
    \makebox[0.19\linewidth][c]{   \small\textsc{Mixture}} \\[2pt]
    \includegraphics[width=0.19\linewidth]{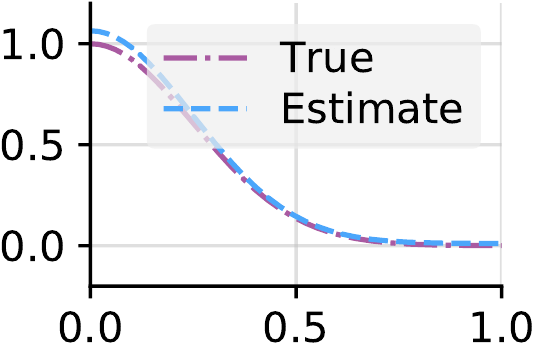}%
    \hfill%
    \includegraphics[width=0.19\linewidth]{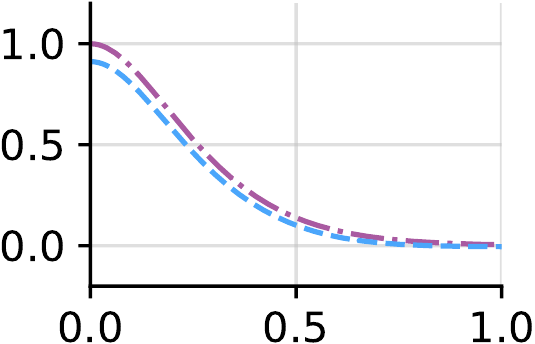}%
    \hfill%
    \includegraphics[width=0.19\linewidth]{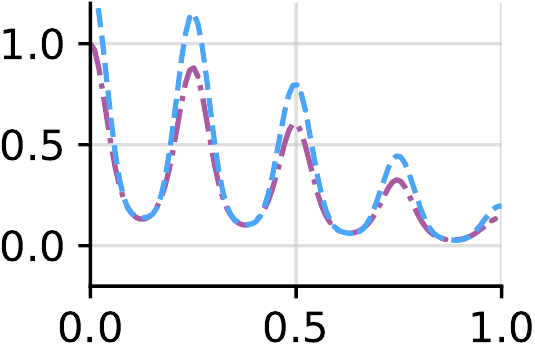}%
    \hfill%
    \includegraphics[width=0.19\linewidth]{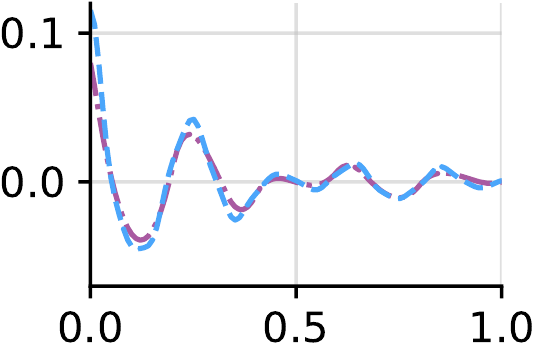}%
    \hfill%
    \includegraphics[width=0.19\linewidth]{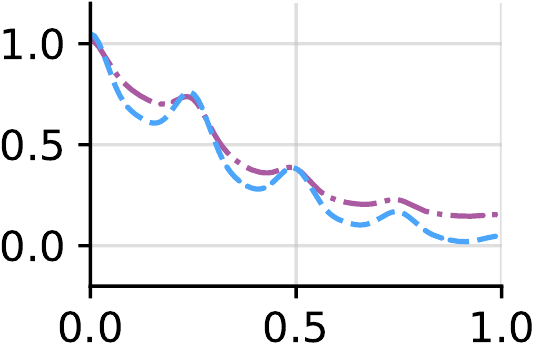}\\
    \vspace*{-5pt}
    \caption{
        Learned and true stationary prior covariance functions in the 1D experiments.
        The prior covariance function can be extracted from the model by taking $D\us{(c)} = \es$.
    }
    \label{fig:1D_kernels}
\end{figure}

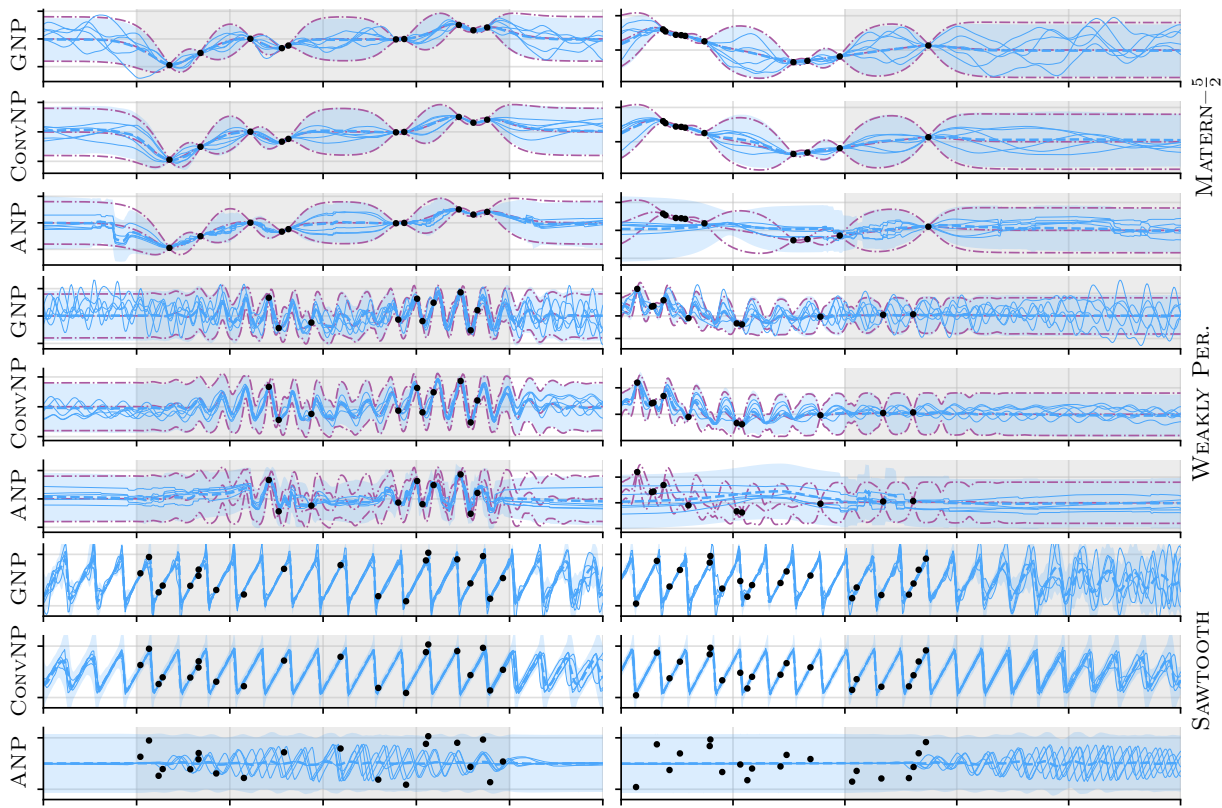
\begin{figure}[t]
    \centering
    \begin{tikzpicture}[overlay, remember picture]
        \node [rotate=90, anchor=south] () at (0, 17pt) {\scriptsize\textsc{GNP}};
    \end{tikzpicture}%
    \includegraphics[width=0.496\linewidth]{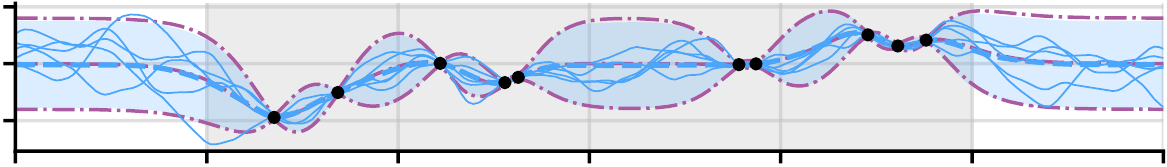}
    \hfill \includegraphics[width=0.496\linewidth]{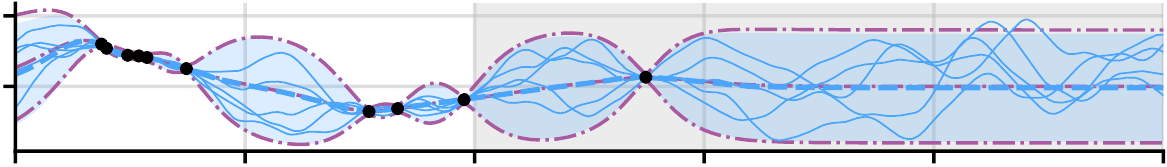}\\[3pt]
    \begin{tikzpicture}[overlay, remember picture]
        \node [rotate=90, anchor=south] () at (0, 17pt) {\scriptsize\textsc{ConvNP}};
    \end{tikzpicture}%
    \includegraphics[width=0.496\linewidth]{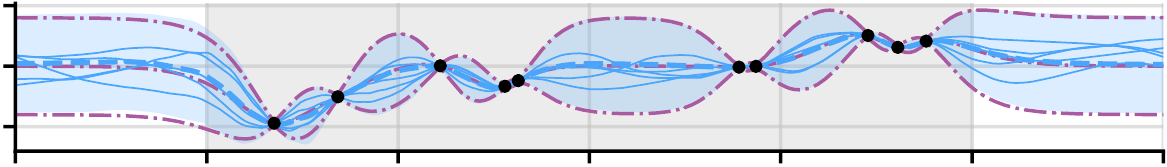}
    \hfill \includegraphics[width=0.496\linewidth]{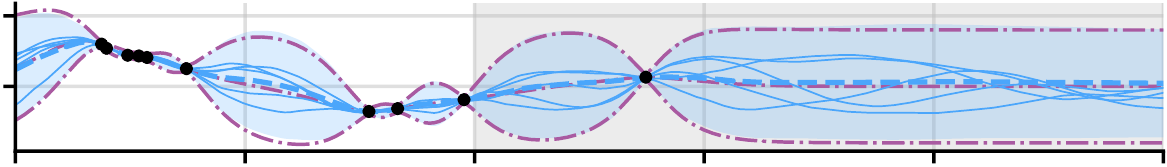}%
    \begin{tikzpicture}[overlay, remember picture]
        \node [rotate=90, anchor=north] () at (0, 17pt) {\footnotesize\textsc{Matern--$\frac52$}};
    \end{tikzpicture}\\[3pt]
    \begin{tikzpicture}[overlay, remember picture]
        \node [rotate=90, anchor=south] () at (0, 17pt) {\scriptsize\textsc{ANP}};
    \end{tikzpicture}%
    \includegraphics[width=0.496\linewidth]{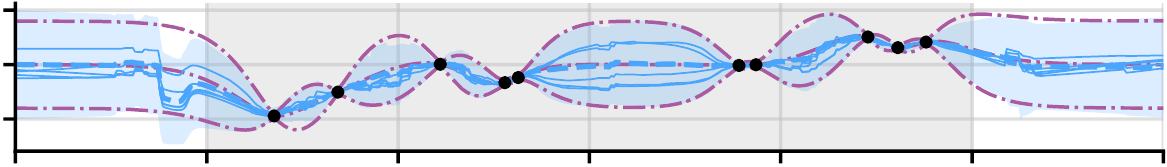}
    \hfill \includegraphics[width=0.496\linewidth]{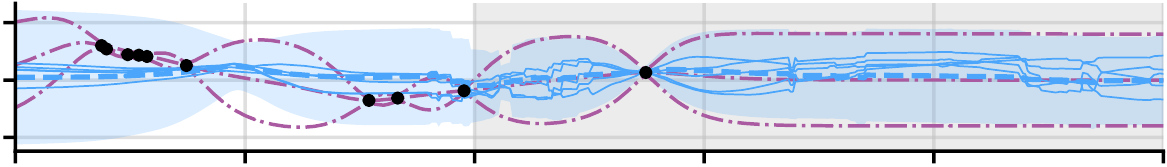}\\
    \begin{tikzpicture}[overlay, remember picture]
        \node [rotate=90, anchor=south] () at (0, 17pt) {\scriptsize\textsc{GNP}};
    \end{tikzpicture}%
    \includegraphics[width=0.496\linewidth]{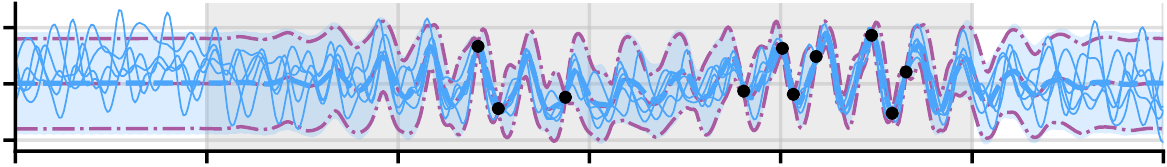}
    \hfill \includegraphics[width=0.496\linewidth]{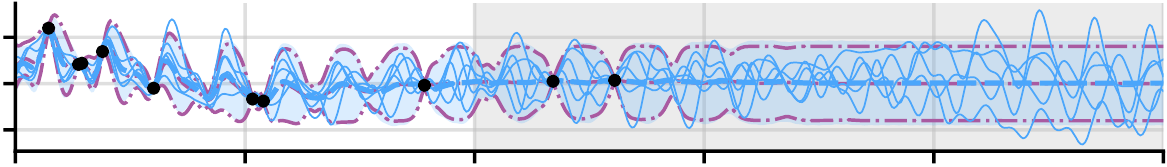}\\[3pt]
    \begin{tikzpicture}[overlay, remember picture]
        \node [rotate=90, anchor=south] () at (0, 17pt) {\scriptsize\textsc{ConvNP}};
    \end{tikzpicture}%
    \includegraphics[width=0.496\linewidth]{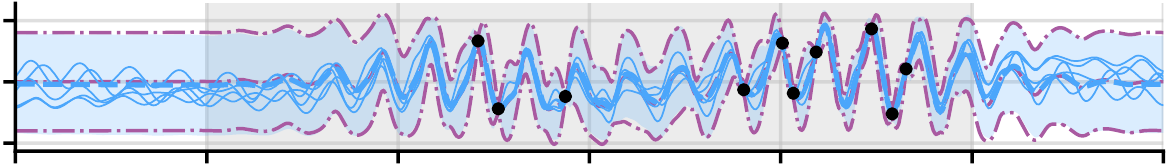}
    \hfill \includegraphics[width=0.496\linewidth]{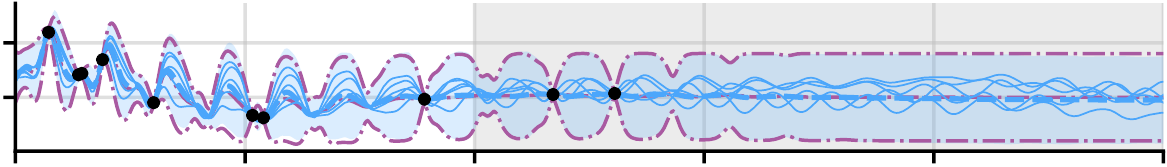}%
    \begin{tikzpicture}[overlay, remember picture]
        \node [rotate=90, anchor=north] () at (0, 17pt) {\footnotesize\textsc{Weakly Per.}};
    \end{tikzpicture}\\[3pt]
    \begin{tikzpicture}[overlay, remember picture]
        \node [rotate=90, anchor=south] () at (0, 17pt) {\scriptsize\textsc{ANP}};
    \end{tikzpicture}%
    \includegraphics[width=0.496\linewidth]{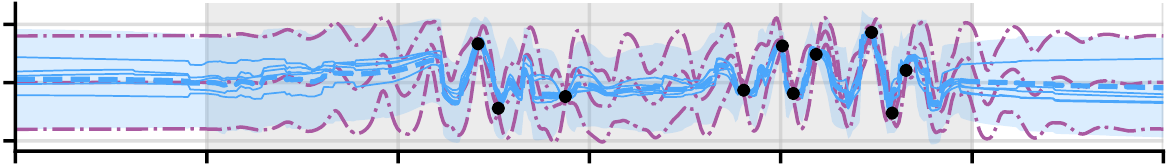}
    \hfill \includegraphics[width=0.496\linewidth]{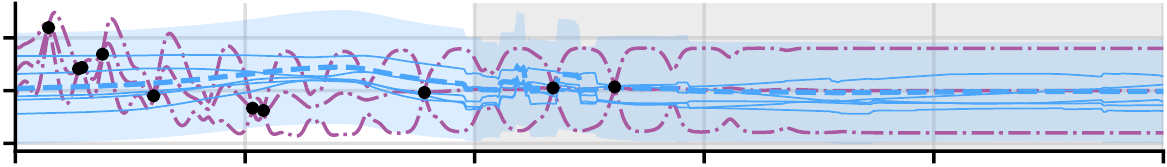}\\
    \begin{tikzpicture}[overlay, remember picture]
        \node [rotate=90, anchor=south] () at (0, 17pt) {\scriptsize\textsc{GNP}};
    \end{tikzpicture}%
    \includegraphics[width=0.496\linewidth]{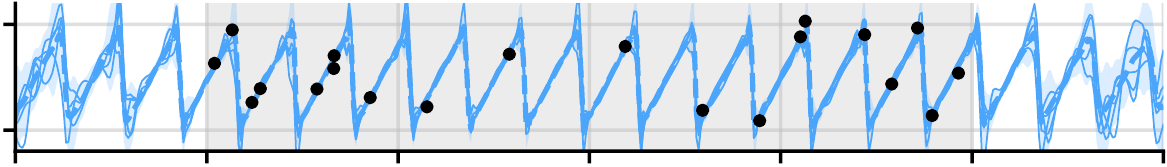}
    \hfill \includegraphics[width=0.496\linewidth]{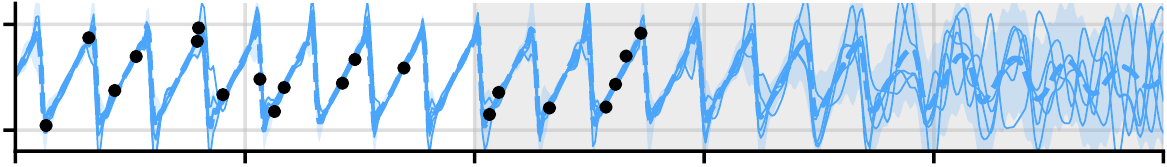}\\[3pt]
    \begin{tikzpicture}[overlay, remember picture]
        \node [rotate=90, anchor=south] () at (0, 17pt) {\scriptsize\textsc{ConvNP}};
    \end{tikzpicture}%
    \includegraphics[width=0.496\linewidth]{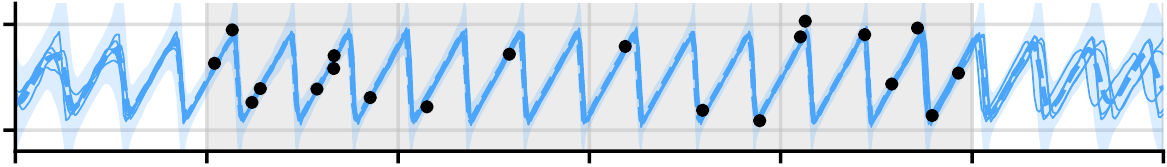}
    \hfill \includegraphics[width=0.496\linewidth]{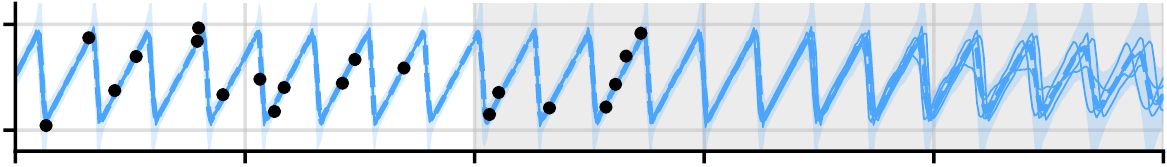}%
    \begin{tikzpicture}[overlay, remember picture]
        \node [rotate=90, anchor=north] () at (0, 17pt) {\footnotesize\textsc{Sawtooth}};
    \end{tikzpicture}\\[3pt]
    \begin{tikzpicture}[overlay, remember picture]
        \node [rotate=90, anchor=south] () at (0, 17pt) {\scriptsize\textsc{ANP}};
    \end{tikzpicture}%
    \includegraphics[width=0.496\linewidth]{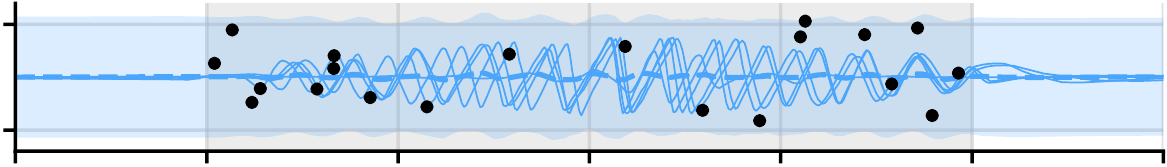}
    \hfill \includegraphics[width=0.496\linewidth]{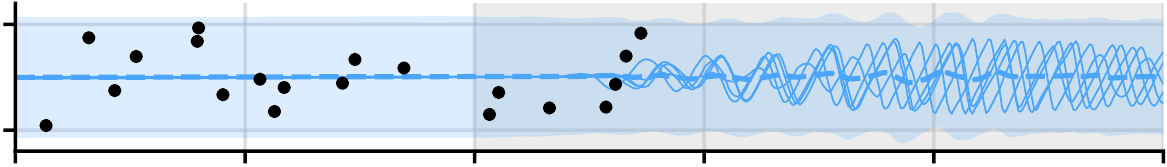}\\[3pt]
    \vspace*{-5pt}
    \caption{
        Sample predictions of models trained in the 1D experiments.
        Grey regions indicate where the models were trained;
        dashed purple lines show optimal predictions by the ground truth.
    }
    \label{fig:1D_plots}
\end{figure}

\sectioncompressed{Experiments}
\label{sec:experiments}

We evaluate the Gaussian Neural Process on synthetic 1D regression experiments.
We follow the experimental setup from \citet{Foong:2020:Meta-Learning_Stationary_Stochastic_Process_Prediction};
see \cref{app:experiments} for more details.
In line with the theoretical analysis from \cref{sec:objective-function-analysis}, but unlike \citet{Foong:2020:Meta-Learning_Stationary_Stochastic_Process_Prediction}, we contaminate all data samples with $\Normal(0, 0.05^2)$-noise.
Code that implements the \textsc{GNP} and reproduces all experiments can be found at \url{https//github.com/wesselb/NeuralProcesses.jl}.

\newcommand{\GNP}{\textsc{GNP}}
\Cref{tab:1D_results} shows the performance of the ground-truth GP (where applicable), the ground-truth GP without correlations, the \textsc{GNP}, the \textsc{ConvCNP} \citep{Gordon:2020:Convolutional_Conditional_Neural_Processes}, the \textsc{ConvNP} \citep{Foong:2020:Meta-Learning_Stationary_Stochastic_Process_Prediction}, and the \textsc{ANP} \citep{Kim:2019:Attentive_Neural_Processes} in an interpolation setup;
\cref{fig:1D_plots} shows samples of the learned models.
The \GNP\ significantly outperforms all other models on all tasks except \textsc{Sawtooth}.
On \textsc{EQ}, and \textsc{Mat\'ern--$\frac52$}, the \GNP\ even achieves parity with the ground-truth GP.
Moreover, the \GNP\ is the best performing model on \textsc{Mixture}, which is highly non-Gaussian, demonstrating that the \GNP\ can successfully approximate non-Gaussian processes.
\Cref{fig:1D_kernels} shows the stationary prior covariance functions learned by the \GNP, which are close to the ground-truth on the GP tasks;
this, together with the likelihood numbers, empirically validates the ability of the \GNP~to recover the prediction map of a ground-truth GP.
Note that the learned covariance functions also match the truth for the non-Gaussian tasks \textsc{Sawtooth} and \textsc{Mixture}.
For \textsc{Sawtooth}, the likelihood of the \GNP\ is worse than the \textsc{ConvCNP} and only improves over the \textsc{ANP}, which demonstrates that, on certain non-Gaussian tasks, non-Gaussian approximations, like the \textsc{ConvNP}, can offer substantially better performance.
On the most expensive tasks (\textsc{Sawtooth} and \textsc{Mixture}), an epoch for the \GNP\ took roughly six times longer than any other model; see \cref{tab:1D_epoch_timings}.
More results are in \cref{app:experiments}, including results for setups that test generalisation and extrapolation performance;
like the \textsc{ConvCNP} and \textsc{ConvNP}, due to translation equivariance, these results shows that the \GNP\ demonstrates excellent ability to generalise.

\sectioncompressed{Conclusion and Future Work}
We have provided a rigorous analysis of the standard ML objective used to train NPs.
Moreover, we propose a new member to the Neural Process family called the Gaussian Neural Process (\textsc{GNP}), which incorporates translation equivariance, provides universal approximation guarantees, and demonstrates encouraging performance in preliminary experiments.
In future work, we aim to investigate incorporating $\Pi\us{p.s.d.}$ (\cref{sec:GNP}) directly into the architecture.

\clearpage
\acks{
    The authors thank Piet Lammers for insightful discussions about stochastic processes and
    David Burt and Cozmin Ududec for helpful comments on a draft.
}
\bibliography{bibliography}
\appendix

\numberwithin{table}{section}
\numberwithin{figure}{section}
\numberwithin{equation}{section}

\clearpage
\section{Notation and Terminology}
\label{app:notation}

\hspace{\parindent}
\textsc{\bfseries Vectors and matrices:}
Denote vectors $\vx$ with boldface lowercase letters and matrices $\mA$ with boldface uppercase letters.
For a vector $\vx$, let $\abs{\vx}$ be its length.
For two matrices $\mA$ and $\mB$, the notation $\mA \succ \mB$ means that $\mA - \mB$ is strictly positive definite.
\vspace{0.25em}

\textsc{\bfseries Observations and data sets:}
Let $\mathcal{X} = \R$ be the space of \emph{inputs} and $\mathcal{Y} = \R$ be the space of \emph{outputs}.
Call a tuple $(x, y) \in \mathcal{X} \times \mathcal{Y}$ an \emph{observation}.
Let $\mathcal{D}_n = (\mathcal{X}\times\mathcal{Y})^n$ be the space of all collections of $n$ observations,
and let $\mathcal{D} = \union_{n=0}^\infty \mathcal{D}_n$ be the space of all finite collections of observations.
Call an element $D \in \D$ a \emph{data set}.
Note that $\es \in \mathcal{D}$.
If $D = \set{(x_i, y_i)}_{i=1}^n \in \D_n$, then denote $D = (\vx, \vy)$ where $\vx = (x_1, \ldots, x_n)$ and $\vy = (y_1, \ldots, y_n)$.
Endow $\mathcal{D}$ with the metric
\begin{equation}
    d_\mathcal{D}(D_1, D_2) = \begin{cases}
        \norm{\vx_1 - \vx_2}_2 + \norm{\vy_1  - \vy_2}_2 & \text{if $\abs{\vx_1}=\abs{\vx_2}$}, \\
        \infty & \text{otherwise}
    \end{cases}
\end{equation}
where $D_1 = (\vx_1, \vy_1)$ and $D_2 = (\vx_2, \vy_2)$.
\vspace{0.25em}

\textsc{\bfseries Probability distributions:}
For all $n \in \N$, let $\mathcal{P}^n$ be the collection of all distributions on $\R^n$, and let $\mathcal{P}\ss{G}^n$ be the collection of all such Gaussian distributions.
Let $C(\mathcal{X}, \mathcal{Y})$ be the collection of continuous functions $\mathcal{X} \to \mathcal{Y}$ endowed with the metric
\begin{equation}
    d_C(f, g) = \sum_{n=1}^\infty 2^{-n} \sup_{x \in [-n, n]}(\abs{f(x) - g(x)} \land 1),
\end{equation}
which makes it a complete and separable metric space.
This metric metricises the topology of \emph{compact convergence}.
Let $\mathcal{P}$ be the space of all probability measures on $(\mathcal{Y}^\mathcal{X}, \mathcal{B}^\mathcal{X})$ where $\mathcal{B}$ is the usual Borel $\sigma$-algebra on $\mathcal{Y}$ and $\mathcal{B}^\mathcal{X}$ is the cylindrical $\sigma$-algebra on $C(\mathcal{X}, \mathcal{Y})$.
Similarly, let $\mathcal{P}\ss{G}$ be the collection of all such Gaussian measures.
Say that a distribution has \emph{full support} if every open set has positive measure.
Let $I = \union_{n=1}^\infty \X^n$ be the collection of all finite index sets or equivalently all finite sets of inputs.
For $\vx \in I$, let $P_\vx$ be the projection on the coordinates $\vx$: $P_\vx f = (f(\vx_1), \ldots, f(\vx_{\abs{\vx}}))$.
For $(\mu_i)_{i \ge 1} \sub \mathcal{P}$ and $\mu \in \mathcal{P}$, say that $(\mu_i)_{i \ge 1}$ converges weakly to $\mu$, denoted $\mu_i \weakto \mu$, if $\mu_i(L) \to \mu(L)$ for all $L \colon C(\X, \Y) \to \R$ continuous and bounded.
(Recall that $\mu(L)$ denotes $\E_\mu[L]$.)
\vspace{0.25em}

\textsc{\bfseries Lebesgue class:}
For all $n \in \N$, let $\mathcal{P}_\lambda^n \sub \mathcal{P}^n$ be the collection of all distributions on $\R^n$ that admit a density with respect to the Lebesgue measure.
Let $\mathcal{P}_\lambda \sub \mathcal{P}$ be the collection of processes where every finite-dimensional distribution admits a density with respect to the Lebesgue measure.
\vspace{0.25em}

\textsc{\bfseries Degeneracy:}
Call a distribution $\mu \in \mathcal{P}^n$ \emph{non-degenerate} if it has a covariance matrix and the covariance matrix is strictly positive definite.
Similarly, call a measure $\mu \in \mathcal{P}$ non-degenerate if every finite-dimensional distribution has a covariance matrix and all those covariance matrices are strictly positive definite.
\vspace{0.25em}

\textsc{\bfseries Gaussianisation:}
For a distribution $\mu \in \mathcal{P}^n$, let $\Normal(\mu) \in \mathcal{P}^n\ss{G}$ be the $n$-dimensional Gaussian distribution with mean equal to the mean of $\mu$ and covariance matrix equal to the covariance matrix of $\mu$, assuming that the latter exists.
Similarly, for a measure $\mu \in \mathcal{P}$, let $\GP(\mu) \in \mathcal{P}\ss{G}$ be the Gaussian process with mean function equal to the mean function of $\mu$ and covariance function equal to the covariance function of $\mu$, assuming that the latter exists.
\vspace{0.25em}

\clearpage
\section{The Gaussian Divergence}
\label{app:Gaussian_divergence}
For two measures $\mu \in \mathcal{P}$ and $\nu \in \mathcal{P}$, their Kullback--Leibler divergence enjoys \emph{positive definiteness}: $\KL(\mu, \nu) \ge 0$ with equality if and only if $\mu = \nu$.
However, when restricting to only Gaussian measures $\nu \in \mathcal{P}\ss{G}$, this property cannot be used anymore, because even the best Gaussian approximation $\nu$ may not achieve $\KL(\mu, \nu) = 0$.
To get around this, we define a divergence induced by the Kullback--Leibler divergence called the \emph{Gaussian divergence}.

For an arbitrary probability distribution $\mu$ on $\R^n$ and a Gaussian distribution $\nu$ on $\R^n$, define their Gaussian divergence $\G(\mu, \nu)$ by
\begin{equation} \label{eq:G}
    \G(\mu, \nu) =
    \begin{cases}
        \KL(\mu, \nu) - \inf_{\xi \in \mathcal{P}^n\ss{G}} \KL(\mu, \xi)
        & \text{if $\KL(\mu, \xi) < \infty$ for some $\xi \in \mathcal{P}^n\ss{G}$,} \\
        \infty & \text{otherwise.}
    \end{cases}
\end{equation}
Let $\vm_1$ and $\vm_2$ be the means of $\mu$ and $\nu$ respectively and let $\mK_1$ and $\mK_2$ be their covariances.
Assume that $\mu$ has a density with respect to the Lebesgue measure and that $\mK_1 \succ 0$ and $\mK_2 \succ 0$.
Then a quick computation shows that
\begin{equation} \label{eq:expression_G}
    \G(\mu, \nu)
    =
        \frac12 d_\mK(\mK_1, \mK_2)
        + \frac12 d_{\vm}(\vm_1, \vm_2;\mK_2)
    = \KL(\Normal(\mu), \nu)
\end{equation}
where $\Normal(\mu) = \Normal(\vm_1, \mK_1)$ and
\begin{align}
    d_{\vm}(\vm_1, \vm_2;\mK_2)
        &= (\vm_1 - \vm_2)^\T \mK_2^{-1}(\vm_1 - \vm_2) \ge 0, \\
    d_\mK(\mK_1, \mK_2)
        &= \log \frac{\abs{\mK_2}}{\abs{\mK_1}}
        + \tr(\mK_2^{-1} \mK_1)
        - n \ge 0.
\end{align}
Since the right-hand side of \eqref{eq:expression_G} is a Kullback--Leibler divergence,
the Gaussian divergence inherits positive definiteness:
if $\mu \in \mathcal{P}_\lambda^n$ and $\nu \in \mathcal{P}\ss{G}^n$ are non-degenerate,
then $\G(\mu, \nu) \ge 0$ with equality if and only if $\nu = \Normal(\mu)$.
The Gaussian divergence inherits more properties from the Kullback--Leibler divergence.
An important inherited property that we will make use of is \emph{monotonicity}.
Let $S\colon \R^n \to \R^{m}$ where $m \le n$ be a projection onto a subset of the coordinates.
Then it is true that
$
    \G(S(\mu), S(\nu))
    \le \G(\mu, \nu).
$

\begin{thm}[\citet{Sun:2018:Functional_Variational_Bayesian_Neural_Networks}]
    Let $\mu \in \mathcal{P}$ and $\nu \in \mathcal{P}$.
    Then
    \begin{equation}
        \KL(\mu, \nu) = \sup_{\vx \in I} \, \KL(P_\vx \mu, P_\vx \nu).
    \end{equation}
\end{thm}

We take inspiration from the above characterisation of the Kullback--Leibler divergence between stochastic processes by \citet{Sun:2018:Functional_Variational_Bayesian_Neural_Networks} to extend the definition of the Gaussian divergence over distributions on $\R^n$ to arbitrary probability measures in $\mathcal{P}$
by taking the supremum over index sets:
\begin{equation}
    \G(\mu, \nu)
    = \sup_{\vx \in I}\, \G(P_{\vx} \mu, P_{\vx} \nu).
\end{equation}
Suppose that $\mu$ and $\nu$ are non-degenerate and that every finite-dimensional distribution of $\mu$ has a density with respect to the Lebesgue measure.
Then \eqref{eq:expression_G} shows the equality
\begin{equation}
    \G(P_{\vx} \mu, P_{\vx} \nu)
    = \KL(\Normal(P_{\vx}\mu), P_{\vx} \nu)
    = \KL(P_{\vx}\, \GP(\mu), P_{\vx} \nu).
\end{equation}
Therefore, taking the supremum over $\vx \in I$, we find that $\G(\mu, \nu) = \KL(\GP(\mu), \nu)$, \cf~\eqref{eq:expression_G}.
The right-hand side is a Kullback--Leibler divergence.
This means that the Gaussian divergence between processes also inherits positive definiteness:
if $\mu \in \mathcal{P}_\lambda$ and $\nu \in \mathcal{P}\ss{G}$ are non-degenerate,
then $\G(\mu, \nu) \ge 0$ with equality if and only if $\nu = \GP(\mu)$.

Let $\mu \in \mathcal{P}^n_\lambda$ and $\nu \in \mathcal{P}^n\ss{G}$ be non-degenerate.
Then \eqref{eq:expression_G} and the definition \eqref{eq:G} show that
\begin{equation} \label{eq:KL_inf}
    \KL(\mu, \Normal(\mu))
    = \inf_{\xi \in \mathcal{P}\ss{G}^n} \KL(\mu, \xi).
\end{equation}
Therefore, if $\mu \in \mathcal{P}^n_\lambda$ and $\nu \in \mathcal{P}^n\ss{G}$ are non-degenerate, then we can also write
\begin{equation} \label{eq:Gn_identification}
    \G(\mu, \nu) = \KL(\mu, \nu) - \KL(\mu, \Normal(\mu)).
\end{equation}
Since $\KL(\mu, \Normal(\mu))$ is a constant that does not depend on $\nu$, we have the following result.

\begin{prop} \label{prop:consistent_fdds}
    Let $\mu \in \mathcal{P}^n_\lambda$ be non-degenerate and assume that, for every $\vx \in I$, there exists some non-degenerate $\mu^{\vx}\ss{G} \in \mathcal{P}\ss{G}^{\abs{\vx}}$ such that $\KL(P_\vx \mu, \mu^{\vx}\ss{G}) < \infty$.
    Then
    \begin{equation}
        \argmin_{\nu \in \mathcal{P}^{\abs{\vx}}\ss{G}}\, \KL(P_\vx \mu, \nu)
        = \argmin_{\nu \in \mathcal{P}^{\abs{\vx}}\ss{G}}\, \G(P_\vx \mu, \nu)
        = \Normal(P_\vx \mu)
        = P_\vx \,\GP(\mu)
        \quad \text{for all $\vx \in I$}.
    \end{equation}
\end{prop}
\begin{proof}
    Follows from \eqref{eq:Gn_identification} if $\KL(P_\vx \mu, \Normal(P_\vx \mu)) < \infty$ for all $\vx \in I$,
    which, in turn, follows from \eqref{eq:KL_inf} and the assumption.
\end{proof}

Equality \eqref{eq:Gn_identification} was key in the proof of \cref{prop:consistent_fdds}.
We proceed to develop a similar expression for the Gaussian divergence between processes:
\cref{thm:G_identification}.
We first address the issue that $\KL(\mu, \GP(\mu))$ may be infinite.

\begin{prop} \label{prop:G_inequality}
    Let $\mu \in \mathcal{P}_\lambda$ and $\nu \in \mathcal{P}\ss{G}$ be non-degenerate.
    Then $\KL(\mu, \GP(\mu)) \le \KL(\mu, \nu)$.
\end{prop}
\begin{proof}
    Let $\vx \in I$.
    Using that $\G(P_\vx \mu, P_\vx\, \GP(\mu)) = 0$,
    \begin{equation}
        \KL(P_\vx \mu, P_\vx\, \GP(\mu))
        = \inf_{\xi \in \mathcal{P}\ss{G}^n} \KL(P_\vx \mu, \xi).
    \end{equation}
    Therefore,
    \begin{equation}
        \KL(P_\vx \mu, P_\vx\, \GP(\mu)) \le \KL(P_\vx \mu, P_\vx \nu),
    \end{equation}
    and take the supremum over $\vx \in I$ to conclude.
\end{proof}

As as consequence, $\KL(\mu, \GP(\mu)) < \infty$ if and only if there exits some non-degenerate $\mu\ss{G} \in \mathcal{P}\ss{G}$ such that $\KL(\mu, \mu\ss{G}) < \infty$.

\begin{thm} \label{thm:G_identification}
    Let $\mu \in \mathcal{P}_\lambda$ be non-degenerate and assume that there exits some non-degenerate $\mu\ss{G} \in \mathcal{P}\ss{G}$ such that $\KL(\mu, \mu\ss{G}) < \infty$.
    Let $\nu \in \mathcal{P}\ss{G}$ be non-degenerate.
    Then $\KL(\mu, \nu)$ and $\G(\mu, \nu)$ differ by a finite constant that only depends on $\mu$:
    \begin{equation}
        \KL(\mu, \nu) - \KL(\mu, \GP(\mu)) = \G(\mu, \nu).
    \end{equation}
\end{thm}
\begin{proof}
    By the assumptions and \cref{prop:G_inequality}, $\KL(\mu, \GP(\mu)) < \infty$.
    Let $\vx \in I$.
    Note that
    \begin{equation}
        \KL(P_\vx \mu, P_\vx \nu) - \KL(\mu, \GP(\mu))
        \le \G(P_\vx \mu, P_\vx \nu).
    \end{equation}
    Take the supremum over $\vx \in I$ to find ``$\le$''.

    For the reverse inequality, write
    \begin{equation}
        \inf_{\vx \in I}\, (\KL(\mu, \nu) - \KL(P_\vx \mu, P_\vx\, \GP(\mu))
        = \KL(\mu, \nu) - \KL(\mu, \GP(\mu)).
    \end{equation}
    Let $\vx \in I$.
    For every $\e > 0$, can find a $\vx' \in I$ such that
    \begin{equation}
        \KL(\mu, \nu) - \KL(P_{\vx'} \mu, P_{\vx'} \GP(\mu))
        < \KL(\mu, \nu) - \KL(\mu, \GP(\mu)) + \e.
    \end{equation}
    Then, by monotonicity of the Kullback--Leibler divergence,
    \begin{equation}
        \KL(\mu, \nu) - \KL(P_{\vx\cup\vx'} \mu, P_{\vx\cup\vx'} \GP(\mu))
        < \KL(\mu, \nu) - \KL(\mu, \GP(\mu)) + \e.
    \end{equation}
    Therefore, bounding $\KL(P_{\vx\cup\vx'} \mu, P_{\vx\cup\vx'} \nu) \le \KL(\mu, \nu)$,
    \begin{equation}
        \G(P_{\vx\cup\vx'} \mu, P_{\vx\cup\vx'} \nu)
        < \KL(\mu, \nu) - \KL(\mu, \GP(\mu)) + \e.
    \end{equation}
    Hence, by monotonicity of the Gaussian divergence,
    \begin{equation}
        \G(P_{\vx} \mu, P_{\vx} \nu)
        < \KL(\mu, \nu) - \KL(\mu, \GP(\mu)) + \e.
    \end{equation}
    Let $\e \downarrow 0$ and take the supremum over $\vx \in I$ to conclude.
\end{proof}

\begin{cor} \label{cor:existence_uniqueness}
    Let $\mu \in \mathcal{P}_\lambda$ be non-degenerate and assume that there exits some non-degenerate $\mu\ss{G} \in \mathcal{P}\ss{G}$ such that $\KL(\mu, \mu\ss{G}) < \infty$.
    Then
    \begin{equation}
        \argmin_{\nu \in \mathcal{P}\ss{G}}\, \KL(\mu, \nu)
        = \argmin_{\nu \in \mathcal{P}\ss{G}}\, \G(\mu, \nu)
        = \GP(\mu).
    \end{equation}
\end{cor}

\clearpage
\section{Noisy Processes and Prediction Maps}
\label{app:noisy_processes}

\subsection{Noisy Processes}

\begin{dfn}[Noisy Process] \label{def:noisy_process}
    Call a stochastic process $f$ \emph{noisy with noise variance $\sigma^2 > 0$} if it a sum $f = f\us{s} + f\us{n}$ of two independent processes $f\us{s}$ and $f\us{n}$ where $f\us{s}$ is a continuous process called the \emph{smooth part} and $f\us{n}$ is such that $P_\vx f\us{n} \sim \Normal(\vnull, \sigma^2 \mI)$ for all $\vx \in I$ and called the \emph{noisy part}.
    We indicate that a collection of processes is noisy by adding a bar $\overline{\phantom{\mathcal{P}}}$: $\overline{\mathcal{P}}$ and $\overline{\mathcal{P}}\ss{G}$.
\end{dfn}

Before anything else, we check that the definition is well posed.

\begin{prop}[Noisy Processes are Well Defined]
    Consider two noisy processes $f_1 = f_1\us{s} + f_1\us{n}$ and $f_2 = f_2\us{s} + f_2\us{n}$.
    If $f_1 \disteq f_2$, then $f_1\us{s}\disteq f_2\us{s}$ and $f_1\us{n}\disteq f_2\us{n}$.
\end{prop}
\begin{proof}
    Let $\vx \in I$ have all distinct elements.
    We show that $f_1\us{s}(\vx)\disteq f_2\us{s}(\vx)$ and $f_1\us{n}(\vx)\disteq f_2\us{n}(\vx)$.
    For $i = 1, 2$, set
    \begin{equation}
        \vy^{(i,n)}
        = \frac1n \sum_{i=1}^n \begin{bmatrix}
            f_i(x_1 - 2^{-n}) \\ \vdots \\ f_i(x_{\abs{\vx}} - 2^{-n})
        \end{bmatrix}
        = \frac1n \sum_{i=1}^n \begin{bmatrix}
            f\us{s}_i(x_1 - 2^{-n}) \\ \vdots \\ f\us{s}_i(x_{\abs{\vx}} - 2^{-n})
        \end{bmatrix} + \frac1n \sum_{i=1}^n \begin{bmatrix}
            f\us{n}_i(x_1 - 2^{-n}) \\ \vdots \\ f\us{n}_i(x_{\abs{\vx}} - 2^{-n})
        \end{bmatrix}.
    \end{equation}
    By continuity of $f\us{s}$ and the Strong Law of Large Numbers, for $i = 1,2$,
    \begin{equation}
        (\vy^{(i,n)}, f_i(\vx)) \to (f_i\us{s}(\vx), f_i(\vx))
        \quad\text{as $n \to \infty$}
    \end{equation}
    almost surely and hence weakly.
    But we assumed that $f_1 \disteq f_2$, so
    \begin{equation}
        (\vy^{(1,n)}, f_1(\vx)) \disteq (\vy^{(2,n)}, f_2(\vx))
        \quad\text{for all $n \ge 1$}.
    \end{equation}
    Therefore, since weak limits are unique,
    \begin{equation}
        (f_1\us{s}(\vx), f_1(\vx)) \disteq (f_2\us{s}(\vx), f_2(\vx)).
    \end{equation}
    Finally, using that $(\vy_1, \vy_2) \mapsto (\vy_1, \vy_2 - \vy_1)$ is continuous, hence measurable, we have
    \begin{equation}
        (f_1\us{s}(\vx), f_1\us{n}(\vx)) \disteq (f_2\us{s}(\vx), f_2\us{n}(\vx)).
    \end{equation}
    In particular, this means that $f_1\us{s}(\vx) \disteq f_2\us{s}(\vx)$ and $f_1\us{n}(\vx) \disteq f_2\us{n}(\vx)$.
\end{proof}

Note that noise variance $\sigma^2 = 0$ is not allowed.
Moreover, note that noisy process do not have measurable sample paths in general.
They are very irregular objects that can only be worked with at a countable collection of indices.
From a theoretical perspective, this is undesirable and raises concerns.
However, from a practical perspective, we will only ever work at a finite collection of indices;
and at every finite collection indices, noisy processes do provide the right model, because in most practical applications observations are contaminated with a little bit of noise.

\begin{prop}[Noisy Processes are Lebesgue Class] \label{prop:density_of_noisy_process}
    Let $f \sim \mu \in \overline{\mathcal{P}}$ be a noisy process with noise variance $\sigma^2 > 0$.
    Let $\vx \in I$ be an index set.
    Then $P_\vx \mu$ has the following density with respect to the Lebesgue measure:
    \begin{equation}
        p^\vx_f(\vy) = \E_{f\us{s}}[\Normal(\vy \cond f\us{s}(\vx), \sigma^2 \mI)].
    \end{equation}
    Consequently, $\overline{\mathcal{P}} \sub \mathcal{P}_\lambda$.
\end{prop}
\begin{proof}
    Let $B \sub \Y^{\abs{\vx}}$ be a Borel set.
    Then
    \begin{align}
        \smash{\int_B p^\vx_f(\vy) \isd \vy}
        &\overset{\text{(i)}}{=} \E_{f\us{s}}[\P_{f\us{n}}(f\us{n}(\vx) \in B - f\us{s}(\vx))] \\
        &= \E_{f\us{s}}[\E_{f\us{n}}[\ind_B(f\us{s}(\vx) + f\us{n}(\vx))]] \\
        &= \P(f\us{s}(\vx) + f\us{n}(\vx) \in B),
    \end{align}
    using in (i) that $f\us{n}(\vx) \sim \Normal(\vnull, \sigma^2 \mI)$.
\end{proof}

\begin{prop}[Equality of Noisy Processes] \label{prop:equality_noisy_processes}
    Let $\mu \in \overline{\mathcal{P}}$ and $\nu \in \overline{\mathcal{P}}$ be two noisy processes with noise variances $\sigma^2_\mu > 0$ and $\sigma^2_\nu > 0$.
    Let $\tilde{I} \sub I$ be dense in $I$.
    If $P_\vx \mu = P_\vx \nu$ for all $\vx \in \tilde{I}$, then $\mu = \nu$.
    Moreover, if $\mu$ and $\nu$ are Gaussian, then it suffices to consider $\tilde{I} \sub \X^n$ dense in $\X^n$ for any fixed $n \ge 2$.
\end{prop}
\begin{proof}
    Let $\vx \in I$.
    Use density of $\tilde I$ to extract a sequence $(\vx_i)_{i \ge 1} \sub \tilde I$ convergent to $\vx$.
    Consider $B \sub \Y^{\abs{\vx}}$ open.
    Let $\vn_\mu \sim \Normal(\vnull, \sigma^2_\mu \mI)$ and $\vn_\nu \sim \Normal(\vnull, \sigma^2_\nu \mI)$.
    Then
    \begin{align}
        P_\vx \mu(B)
        &= \vphantom{\lim_{i \to \infty}} \E_{\mu}[\ind_B(f(\vx) + \vn_\mu)] \\
        &= \lim_{i \to \infty} \E_{\mu}[\ind_B(f(\vx_i) + \vn_\mu)] \\
        &= \lim_{i \to \infty} \E_{\nu}[\ind_B(f(\vx_i) + \vn_\nu)] \\
        &= \vphantom{\lim_{i \to \infty}} \E_{\nu}[\ind_B(f(\vx) + \vn_\nu)] \\
        &= \vphantom{\lim_{i \to \infty}} P_\vx \nu(B).
    \end{align}
    Since $B$ was an arbitrary open set, we conclude that $P_\vx \mu = P_\vx \nu$.

    If $\mu$ and $\nu$ are Gaussian,
    then all finite-dimensional distributions are equal if and only if
    all means and covariances are equal.
    And all means and covariances are equal
    if all finite-dimensional distributions of any fixed size $n \ge 2$ are equal, so it suffices to consider $\tilde{I} \sub \X^n$ dense in $\X^n$ for any fixed $n \ge 2$.
\end{proof}

Since noisy processes do not produce measurable sample paths, we need to define what weak convergence means.

\begin{dfn}[Weak Convergence of Noisy Processes]
    Say that a sequence $(f_i)_{i \ge 1}$ of noisy processes with noise variances $(\sigma_i^2)_{i \ge 1}$ weakly converges to a noisy process $f$ with noise variance $\sigma^2$ if $f\us{s}_i \weakto f\us{s}$ and $\sigma_i^2 \to \sigma^2$.
\end{dfn}

Note that there is no problem with weak convergence of finite-dimensional distributions of noisy processes.
With the above definition, it is easy to check that $f_i \weakto f$ implies that $P_\vx f_i \weakto P_\vx f$ for all $\vx \in I$ where the latter weak convergence is in the usual sense.

In the posterior of a noisy process, the noisy part is renewed.

\begin{dfn}[Posterior of Noisy Process]
    The posterior of a noisy process $f$ given some data $D = (\vx, \vy) \in \mathcal{D}$ is defined by the sum $f\us{s}_D + \hat f\us{n}$ where $f\us{s}_D$ is the posterior of $f\us{s}$ given that $f(\vx) = \vy$ and $\hat f\us{n}$ is an independent copy of $f\us{n}$.
\end{dfn}

If $f\us{s} \sim \mu\us{s} \in \mathcal{P}$ and $f\us{s}_D \sim \mu\us{s}_D \in \mathcal{P}$, then one can show that
\begin{equation}
    \frac
        {\sd \mu_D\us{s}}
        {\sd \mu\us{s}}
    (f\us{s})
    =
        \frac
            {\Normal(\vy; f\us{s}(\vx), \sigma^2 \mI)}
            {\E_{f\us{s}}[\Normal(\vy; f\us{s}(\vx), \sigma^2 \mI)]}.
\end{equation}
Note that $\E_{f\us{s}}[\Normal(\vy; f\us{s}(\vx), \sigma^2 \mI)] > 0$, for otherwise $\abs{f\us{s}(x)} = \infty$ almost surely for some $x \in \X$.
Denote
\begin{equation}
    \pi_f(D) = \mu_D\us{s}
    \quad \text{and} \quad
    \pi'_f(D)
    = \frac
        {\sd \mu_D\us{s}}
        {\sd \mu\us{s}}.
\end{equation}
The map $\pi_f\colon \D \to \overline{\mathcal{P}}$ is called the \emph{posterior prediction map} of $f$.

\subsection{Prediction Maps}

\begin{dfn}[Prediction Map] \label{def:prediction_map}
    A map $\pi\colon \mathcal{D} \to \mathcal{P}$ is called a \emph{prediction map}.
    Call a prediction map \emph{Gaussian} if it maps to Gaussian processes and \emph{noisy} if it maps to noisy processes.
\end{dfn}

\begin{prop}[Continuity in the Data] \label{prop:continuity_in_the_data}
    With probability one, the map $D \mapsto \pi'_f(D)(f\us{s})$ is continuous.
    We call this property \emph{continuity in the data}.
\end{prop}
\begin{proof}
    Follows from continuity of $f\us{s}$ and
    \begin{equation}
        \Normal(\vy; f\us{s}(\vx), \sigma^2 \mI) \le (2 \pi \sigma^2)^{-\frac12 \abs{\vx}}
    \end{equation}
    in combination with bounded convergence.
\end{proof}

\begin{prop}[Local Boundedness] \label{prop:local_boundedness}
    For any compact collection of data sets $\tilde \D \sub \D$,
    \begin{equation}
        0 < \sup_{D \in \tilde \D} \sup_{f\us{s} \in \Y^\X} \pi'_f(D)(f\us{s}) < \infty.
    \end{equation}
    We call this property \emph{local boundedness}.
\end{prop}
\begin{proof}
    Note that
    \begin{equation} \label{eq:M_D}
        M_D
        := \sup_{f\us{s} \in \Y^\X} \pi'_f(D)(f\us{s})
        = \frac
            {(2 \pi \sigma^2)^{-\frac12 \abs{\vx}}}
            {\E_{f\us{s}}[\Normal(\vy; f\us{s}(\vx), \sigma^2 \mI)]},
    \end{equation}
    which by an argument similar to the proof of \cref{prop:continuity_in_the_data} is continuous in $D$.
    Moreover, we have that $0 < M_D < \infty$ for all $D \in \tilde D$.
    Therefore, $0 < \sup_{D \in \tilde \D} M_D < \infty$ by continuity of $D \mapsto M_D$ and compactness of $\tilde \D$.
\end{proof}

\begin{dfn}[Bounded Collection of Data Sets] \label{def:bounded_collection_of_data_sets}
    For $\tilde \D \sub \D$, define
    \[
        \norm{\tilde \D}_\infty = \sup_{(\vx, \vy) \in \tilde\D} \sum_{i=1}^{\abs{\vx}} (\abs{y_i} \lor 1).
    \]
    In particular, if $(\vx, \vy) \in \tilde \D$, then $\abs{\vx} \le \norm{\tilde \D}_\infty$ and $\norm{\vy}_\infty \le \norm{\tilde \D}_\infty$.
    Call a collection of data sets $\tilde \D \sub \D$ \emph{bounded} if $\norm{\tilde \D}_\infty < \infty$.
\end{dfn}

Using this definition of boundedness, we can refine \cref{prop:local_boundedness} to obtain a quantitative bound.

\begin{prop}[Local Boundedness (Cont'd)]
    Assume that
    \begin{equation}
        V = \sup_{x \in \X} \E[(f\us{s}(x))^2] < \infty.
    \end{equation}
    Then, for any bounded collection of data sets $\tilde \D \sub \D$,
    \begin{equation} \label{eq:bounded_data_set_implies_bounded_bound}
        1 \le \sup_{D \in \tilde \D} \sup_{f\us{s} \in \Y^\X} \pi'_f(D)(f\us{s})
        \le 2\exp\parens*{
            \frac{2 \norm{\tilde \D}_\infty^2 V}{\sigma^2} + \frac{\norm{\tilde \D}_\infty^3}{\sigma^2}
        }.
    \end{equation}
\end{prop}
\begin{proof}
    Start out from \eqref{eq:M_D}:
    \begin{equation}
        \frac{1}{M_D} = \E_{f\us{s}}\sbrac*{
            \exp\parens*{
                -\frac{\norm{\vy - f\us{s}(\vx)}^2}{2 \sigma^2}
            }
        }
        \ge \exp\parens*{
                -\frac{\norm{\vy}^2}{\sigma^2}
        } \E_{f\us{s}}\sbrac*{
            \exp\parens*{
                -\frac{\norm{f\us{s}(\vx)}^2}{\sigma^2}
            }
        }
    \end{equation}
    Therefore,
    \begin{equation}
        1 \le M_D \le \E_{f\us{s}}\sbrac*{
            \exp\parens*{
                -\frac{\norm{f\us{s}(\vx)}^2}{\sigma^2}
            }
        }^{-1} \exp\parens*{
            \frac{\norm{\vy}^2}{\sigma^2}
        }.
    \end{equation}
    Now estimate
    \begin{equation}
        \E_{f\us{s}}\sbrac*{
            \exp\parens*{
                -\frac{\norm{f\us{s}(\vx)}^2}{\sigma^2}
            }
        }
        \ge \P(\norm{f\us{s}(\vx)}_\infty \le R)\exp\parens*{
            -\frac{n R^2}{\sigma^2}
        }.
    \end{equation}
    Choose $R = \sqrt{2 n V}$ to obtain
    \begin{align}
        \P(\norm{f\us{s}(\vx)}_\infty \le R)
        \ge 1 - n\,\sup_{x \in \X} \,\P(\abs{f\us{s}(x)} > R)
        \ge 1 - \frac{n V}{R^2} = \frac12.
    \end{align}
    With this choice for $R$, we find
    \begin{equation}
        1 \le M_D
        \le 2\exp\parens*{
            \frac{n R^2}{\sigma^2} + \frac{n \norm{\vy}_\infty^2}{\sigma^2}
        }
        = 2\exp\parens*{
            \frac{2 n^2 V}{\sigma^2} + \frac{n \norm{\vy}_\infty^2}{\sigma^2}
        }.
    \end{equation}
    The result then follows from the observations that $n \le \norm{\tilde \D}_\infty$ and $\norm{\vy}_\infty \le \norm{\tilde \D}_\infty$.
\end{proof}

\begin{dfn}[Continuous Prediction Map] \label{def:continuous_prediction_map}
    Call a prediction map $\pi \colon \mathcal{D} \to \mathcal{P}$ \emph{continuous} if $D_i \to D$ implies that $\pi(D_i) \weakto \pi(D)$.
    Denote the collection of all continuous prediction maps by $\mathcal{M}$.
    Write a subscript $\vphantom{\M}\ss{G}$ if the prediction maps are also Gaussian: $\mathcal{M}\ss{G}$.
    Write a bar $\overline{\phantom{\M}}$ if the prediction maps are also noisy: $\overline{\mathcal{M}}$.
    Call a prediction map $\pi \colon \mathcal{D} \to \mathcal{P}$ \emph{continuous along its finite-dimensional distributions} if $D_i \to D$ implies that $P_\vx \pi(D_i) \weakto P_\vx \pi(D)$ for all $\vx \in I$.
    Write a superscript $\vphantom{\M}\us{f.d.d.}$ if the prediction maps are \emph{only} continuous along their finite-dimensional distributions: $\mathcal{M}\us{f.d.d.}$.
    Note the following inclusions:
    \begin{equation}
        \mathcal{M} \sub \mathcal{M}\us{f.d.d.},
        \quad\mathcal{M}\ss{G} \sub \mathcal{M}\us{f.d.d.}\ss{G},
        \quad\overline{\mathcal{M}} \sub \overline{\mathcal{M}}\us{f.d.d.},
        \quad\overline{\mathcal{M}}\ss{G} \sub \overline{\mathcal{M}}\us{f.d.d.}\ss{G}.
    \end{equation}
\end{dfn}

\begin{prop}[Equality of Continuous Prediction Maps] \label{prop:equality_of_prediction_maps}
    Let $\pi_1 \in \mathcal{M}\us{f.d.d.}$ and $\pi_2 \in \mathcal{M}\us{f.d.d.}$.
    Let $\tilde \D \sub \D$.
    If $\pi_1 = \pi_2$ are equal on a dense subset of $\tilde \D$, then $\pi_1 = \pi_2$ are equal on all of $\tilde \D$.
\end{prop}
\begin{proof}
    Let $D \in \D$.
    Extract $(D_i)_{i \ge 1} \sub \tilde \D$ convergent to $D$.
    Let $\vx \in I$.
    By the assumed continuity,
    $P_\vx \pi_1(D_i) \weakto P_\vx \pi_1(D)$ and $P_\vx \pi_2(D_i) \weakto P_\vx \pi_2(D)$.
    Let $L\colon \R^{\abs{\vx}} \to \R$ be continuous and bounded.
    Then
    \begin{equation}
        P_\vx \pi_1(D)(L)
        = \lim_{i \to \infty} P_\vx \pi_1(D_i)(L)
        = \lim_{i \to \infty} P_\vx \pi_2(D_i)(L)
        = P_\vx \pi_2(D)(L).
    \end{equation}
    Since $\vx$ and $L$ were arbitrary, $\pi_1(D) = \pi_2(D)$.
\end{proof}

\begin{prop}[Noisy Posterior Prediction Map is Continuous] \label{prop:noisy_posterior_prediction_map_is_continuous}
    Let $f$ be a noisy process and $\pi_f$ the associated posterior prediction map.
    Then $\pi_f \in \mathcal{M}$.
\end{prop}
\begin{proof}
    For all $D \in \D$, the noisy part of $\pi_f(D)$ is equal in distribution.
    Hence, it suffices to show that $D_i \to D$ implies that $f\us{s}_{D_i} \weakto f\us{s}_D$.
    Let $L\colon C(\X, \Y) \to \R$ be continuous and bounded.
    Then
    \begin{equation}
        \E[L(f\us{s}_{D_i})]
        = \E[\pi'(D_i)(f\us{s})L(f\us{s})]
        \to \E[\pi'(D)(f\us{s})L(f\us{s})]
        = \E[L(f\us{s}_{D})]
    \end{equation}
    by continuity in the data (\cref{prop:continuity_in_the_data}) and bounded convergence using local boundedness (\cref{prop:local_boundedness}).
\end{proof}

\begin{prop}[Noisy Posterior Prediction Map is Bounded] \label{prop:posterior_prediction_map_bounded_second_moment}
    Let $f$ be a noisy process and let $\tilde \D \sub \D$ be a bounded collection of data sets.
    Suppose that $\sup_{x \in \X} \E[(f\us{s}(x))^2] < \infty$.
    Then $\sup_{x \in \X,\, D \in \tilde \D} \E[(f_D\us{s}(x))^2] < \infty$.
\end{prop}
\begin{proof}
    Follows from
    \begin{equation}
        \E[(f_D\us{s}(x))^2]
        \le
            \bigg(\sup_{D \in \tilde \D} \sup_{f\us{s} \in \Y^\X} \pi'_f(D)(f\us{s})\bigg)
            \bigg(\sup_{x \in \X} \E[(f\us{s}(x))^2]\bigg)
    \end{equation}
    in combination with \eqref{eq:bounded_data_set_implies_bounded_bound}.
\end{proof}

\subsection{Gaussianised Prediction Maps}

\begin{dfn}[Gaussianised Prediction Map]
    Given a prediction map $\pi\colon \mathcal{D} \to \mathcal{P}$, the \emph{Gaussianised prediction map} $\GP(\pi)$ is defined by $\GP(\pi)(D) = \GP(\pi(D))$.
\end{dfn}

The Gaussianisation of a noisy process $f \sim \mu \in \overline{\mathcal{P}}$ is equal to $f\ss{G}\us{s} + f\us{n}$ where $f\ss{G}\us{s} \sim \GP(\mu\us{s})$.
This is perfectly well defined.
However, a subtle technical issue is that $f\ss{G}\us{s}$ may not be a continuous process, which means that the Gaussianisation of a noisy process is not necessarily a noisy process.
To prevent this from happening, we impose regularity conditions on $f\us{s}$.

\begin{prop} \label{prop:Holder_condition}
    Let $f$ be a noisy process and let $\pi_f$ the associated posterior prediction map.
    Suppose that there exist $p \ge 2$, $\beta \in (\tfrac12, 1]$, a constant $c > 0$ and a radius $r > 0$ such that
    \begin{equation}
        \norm{f\us{s}(x)  - f\us{s}(y)}_{L^p} \le c \abs{x - y}^\beta
        \quad\text{whenever}\quad
        \abs{x - y} < r.
    \end{equation}
    Then, for all $D \in \D$, if $\GP(\pi_f(D))$ exists, it is a noisy process.
\end{prop}
\begin{proof}
    Let $D \in \D$.
    As explained above, assuming that $\GP(\pi_f(D))$ exists, \ie~that $\pi_f(D)$ has a mean function and covariance function, it remains to show that the smooth part of $\GP(\pi_f(D))$ is a continuous process.
    Let $f\us{s}_D$ be the smooth part of $\pi_f(D)$ and let $g\us{s}_{D}$ be the smooth part of $\GP(\pi_f(D))$.
    Since, by construction of $\GP(\pi_f(D))$, the mean functions and covariance functions of $f\us{s}_D$ and $g\us{s}_D$ are equal,
    \begin{equation}
        \E[\abs{g\us{s}_D(x)  - g\us{s}_D(y)}^2]
        = \E[\abs{f\us{s}_D(x)  - f\us{s}_D(y)}^2].
    \end{equation}
    Therefore, using Jensen's Inequality and concavity of $x \mapsto x^{2/p}$ ($p \ge 2$),
    \begin{align}
        \E[\abs{g\us{s}_D(x)  - g\us{s}_D(y)}^2]
        &\le \E[\abs{f\us{s}_D(x)  - f\us{s}_D(y)}^p]^{2/p} \\
        &= \E[\pi'_f(D)(f\us{s}) \abs{f\us{s}(x)  - f\us{s}(y)}^p]^{2/p} \\
        &\le M_D \E[\abs{f\us{s}(x)  - f\us{s}(y)}^p]^{2/p}
    \end{align}
    with $M_D = \sup_{f\us{s} \in \Y^\X} \pi'_f(D)(f\us{s})$.
    By \eqref{eq:M_D}, $0 < M_D < \infty$.
    We can thus continue our sequence of inequalities:
    \begin{equation}
        \E[\abs{g\us{s}_D(x)  - g\us{s}_D(y)}^2]
        \le M_D \norm{f\us{s}(x)  - f\us{s}(y)}^{2}_{L^p}
        \le M_D c^{2} \abs{x - y}^{2 \beta}
    \end{equation}
    whenever $\abs{x - y} < r$.
    Hence,
    \begin{equation}
        \norm{g\us{s}(x)  - g\us{s}(y)}_{L^2}
        \le \sqrt{M_D} c \abs{x - y}^{\beta}
        \quad \text{whenever} \quad
        \abs{x - y} < r.
    \end{equation}
    This shows that $g\us{s}$ satisfies Kolmogorov's Continuity Criterion and thus admits a continuous version.
\end{proof}

Throughout, we assume that the condition from \cref{prop:Holder_condition} always satisfied.
Consequently, the Gaussianisation of any noisy process is always also a noisy process.

\begin{prop}[Gaussianised Noisy Posterior Prediction Map is Cont.] \label{prop:MM_is_continuous}
    Let $f$ be a noisy process and let $\pi_f$ the associated posterior prediction map.
    Suppose that $f\us{s}(x) \in L^2$ for all $x \in \X$.
    Then $\GP(\pi_f) \in \overline{\mathcal{M}}\ss{G}\us{f.d.d.}$.
\end{prop}
\begin{proof}
    Let $D_i \to D$.
    By the assumptions, $\E[f\us{s}(x)]$ exists.
    Hence, using continuity in the data (\cref{prop:continuity_in_the_data}) and dominated convergence in combination with local boundedness (\cref{prop:local_boundedness}),
    \begin{equation}
        \E[f\us{s}_{D_i}(x)]
        = \E[\pi'(D_i)(f\us{s})f\us{s}(x)]
        \to \E[\pi'(D)(f\us{s})f\us{s}(x)]
        = \E[f\us{s}_{D}(x)].
    \end{equation}
    Similarly, $\cov(f\us{s}_{D_i}(x), f\us{s}_{D_i}(y)) \to \cov(f\us{s}_{D}(x), f\us{s}_{D}(y))$.
    Therefore, for all $\vx \in I$, $P_\vx\, \GP(\pi_f)(D_i) \weakto P_\vx\, \GP(\pi_f)(D)$.
\end{proof}

\clearpage
\section{The Objective}
\label{app:objective}

We before discussing the objective, we first get all issues of measurability out of the way.

\begin{prop} \label{prop:measurability_KL}
    Let $\mu_1, \mu_2 \in \overline{\mathcal{P}}$ and $\mu\ss{G} \in \overline{\mathcal{P}}\ss{G}$.
    Fix $n \ge 1$ and consider all $\vx \in \X^n$.
    Then
    \begin{proplist}
        \item \label{prop:measurability_KL_1}
            $\vx \mapsto \KL(P_\vx \mu_1, P_\vx \mu\ss{G})$ is lower semi-continuous, hence measurable;
        \item \label{prop:measurability_KL_2}
            $\vx \mapsto \KL(P_\vx\, \GP(\mu_1), P_\vx \mu\ss{G})$ is lower semi-continuous, hence measurable;
        \item \label{prop:measurability_KL_3}
            $\vx \mapsto \KL(P_\vx \mu_1, P_\vx\, \GP(\mu_2))$ is lower semi-continuous, hence measurable.
    \end{proplist}
\end{prop}
\begin{proof}
    \cref{prop:measurability_KL_2,prop:measurability_KL_3} follow from \cref{prop:measurability_KL_1} by the observations that $\GP(\mu_1) \in \overline{\mathcal{P}}\ss{G} \sub \overline{\mathcal{P}}$ and $\GP(\mu_2) \in \overline{\mathcal{P}}\ss{G}$.
    To prove \cref{prop:measurability_KL_1}, let $(\vx_i)_{i \ge 1} \sub \X^n$ be convergent to $\vx \in \X^n$.
    Then $P_{\vx_i} \mu_1 \weakto P_{\vx} \mu_1$ and $P_{\vx_i} \mu\ss{G} \weakto P_{\vx} \mu\ss{G} $ because the smooth parts of $\mu_1$ and $\mu\ss{G}$ are continuous processes.
    Using that $(\mu, \nu) \mapsto \KL(\mu, \nu)$ is weakly lower semi-continuous \citep{Posner:1975:Random_Coding_Strategies_for_Minimum}, we thus find
    \begin{equation}
        \liminf_{i \to \infty}\, \KL(P_{\vx_i} \mu, P_{\vx_i} \nu) \ge \KL(P_\vx \mu, P_\vx \nu),
    \end{equation}
    which shows that $\vx \mapsto \KL(P_\vx \mu, P_\vx \nu)$ is lower semi-continuous.
\end{proof}

\begin{prop} \label{prop:measurability_E_KL}
    Let $\pi_1, \pi_2 \in \overline{\mathcal{M}}\us{f.d.d.}$ and $\pi\ss{G} \in \overline{\mathcal{M}}\us{f.d.d.}\ss{G}$.
    Suppose that, for all $D \in \D$, the mean functions and covariance functions $\pi_1(D)$ of $\pi_2(D)$ exist.
    Fix $n \ge 1$ and let $p(\vx)$ be a Borel distribution on $\X^n$.
    Then
    \begin{proplist}
        \item \label{prop:measurability_E_KL_1}
            $D \mapsto \E_{p(\vx)}[\KL(P_\vx \pi_1(D), P_\vx \pi\ss{G}(D))]$ is lower semi-continuous, hence measurable;
        \item \label{prop:measurability_E_KL_2}
            $D\mapsto \E_{p(\vx)}[\KL(P_\vx\,\GP(\pi_1)(D), P_\vx \pi\ss{G}(D))]$ is lower semi-continuous, hence measurable;
        \item \label{prop:measurability_E_KL_3}
            $D \mapsto \E_{p(\vx)}[\KL(P_\vx\,\GP(\pi_1)(D), P_\vx\,\GP(\pi_2)(D))]$ is lower semi-continuous, hence measurable.
    \end{proplist}
\end{prop}
\begin{proof}
    By \cref{prop:measurability_KL}, these expectations are all well defined.
    \cref{prop:measurability_E_KL_2,prop:measurability_E_KL_3} follow from \cref{prop:measurability_E_KL_1} by the observations that $\GP(\pi_1) \in \overline{\mathcal{M}}\ss{G}\us{f.d.d.} \sub \overline{\mathcal{M}}\us{f.d.d.}$ and $\GP(\pi_2) \in \overline{\mathcal{M}}\us{f.d.d.}$ (\cref{prop:MM_is_continuous}).
    To prove \cref{prop:measurability_E_KL_1}, let $(D_i)_{i\ge1} \sub \D$ be convergent to $D \in \D$.
    Since $\pi_1 \in \overline{\mathcal{M}}\us{f.d.d.}$ is continuous, $P_\vx \pi_f(D_i) \weakto P_\vx \pi(D)$ for all $\vx \in \X^n$, and the same statement holds for $\pi\ss{G}$.
    Therefore, using Fatou's Lemma and that $(\mu, \nu) \mapsto \KL(\mu, \nu)$ is weakly lower semi-continuous \citep{Posner:1975:Random_Coding_Strategies_for_Minimum}, we thus find
    \begin{align}
        \liminf_{i \to \infty}\, \E_{p(\vx)}[\KL(P_\vx \pi_1(D_i), P_\vx \pi\ss{G}(D_i))]
        &\ge \E_{p(\vx)}[\liminf_{i \to \infty}\, \KL(P_\vx \pi_1(D_i), P_\vx \pi\ss{G}(D_i))] \\
        &\ge \E_{p(\vx)}[\KL(P_\vx \pi_1(D), P_\vx \pi\ss{G}(D))],
    \end{align}
    which shows that $D \mapsto \E_{p(\vx)}[\KL(P_\vx \pi_f(D), P_\vx \pi(D))]$ is lower semi-continuous.
\end{proof}

\begin{prop} \label{prop:bound}
    Let $\mu \in \overline{\mathcal{P}}$ and $\nu \in \overline{\mathcal{P}}\ss{G}$.
    Fix $n \ge 2$.
    Assume the following:
    \begin{enumerate}
        \item[(1)]
            The mean functions and covariance functions of $\mu$ and $\nu$ exist and are uniformly bounded by $M > 0$.
        \item[(2)]
            The processes $\mu$ and $\nu$ are noisy with noise variance greater than $\sigma^2 > 0$.
    \end{enumerate}
    Then
    \begin{equation} \label{eq:bound}
        \KL(P_\vx \mu, P_\vx \nu)
        \le \frac{4 n^2 (M \lor 1)^2}{\sigma^2}.
    \end{equation}
\end{prop}
\begin{proof}
    By the assumption that $\mu$ is noisy with noise variance $\sigma^2 > 0$, for all $\vx \in \X^n$, the finite-dimensional distribution $P_\vx \mu$ has the following density with respect to the Lebesgue measure on $\R^n$ (\cref{prop:density_of_noisy_process}):
    \begin{equation}
        p_\mu^\vx(\vy) = \E_{\mu}[\Normal(\vy; f\us{s}(\vx), \sigma^2)]
    \end{equation}
    For all $\vx \in I_n$, denote
    \begin{equation}
        \Normal(P_\vx \mu) = \Normal(\vm_\mu^\vx, \mK_\mu^\vx)
        \quad\text{and}\quad
        P_\vx \nu = \Normal(\vm_\nu^\vx, \mK_\nu^\vx).
    \end{equation}
    Start out by expanding the Kullback--Leibler divergence:
    \begin{align}
        &\KL(P_\vx \mu, P_\vx \nu) \nonumber \\
        &\quad= \int p_\mu^\vx(\vy)\sbrac*{
            \log p_\mu^\vx(\vy)
            +\frac12(
                \log \,\abs{2 \pi \mK_\nu^\vx}
                + (\vy - \vm_\nu^\vx)^\T (\mK_\nu^\vx)^{-1} (\vy - \vm_\nu^\vx)
            )
        }\isd \vy.
    \end{align}
    Bound
    \begin{equation}
        p_\mu^\vx(\vy)
        \le \sup_{\vy'}\, \Normal(\vy; \vy', \sigma^2)
        \le (2 \pi \sigma^2)^{-\tfrac12 n}.
    \end{equation}
    Therefore, computing the rest of the expectation in closed form,
    \begin{equation}
        \!\!
        \KL(P_\vx \mu, P_\vx \nu)
        \le \frac{1}{2}\bigg(
            \underbracket{
                \log\frac{\abs{\mK_\nu^\vx}}{\sigma^{2n}}
                + \tr((\mK_\nu^\vx)^{-1} \mK_\mu^\vx)
            }_{\text{(i)}}
            +
            \underbracket{
                \vphantom{\log\frac{\abs{\mK_\nu^\vx}}{\sigma^{2n}}}
                (\vm_\mu^\vx - \vm_\nu^\vx)^\T (\mK_\nu^\vx)^{-1} (\vm_\mu^\vx - \vm_\nu^\vx)
            }_{\text{(ii)}}
        \bigg).
    \end{equation}
    We separately bound (i) and (ii).

    For (i), we use Von Neumann's Trace Inequality:
    for any two $n \times n$ positive semi-definite matrices $\mA$ and $\mB$, it holds that
    \begin{equation}
        \tr(\mA \mB) \le \sum_{i=1}^n \gamma_i(\mA) \gamma_i(\mB).
    \end{equation}
    Using this inequality,
    \begin{equation}
        \text{(i)}
        \le \sum_{i=1}^n \parens*{
            \log \frac{\gamma_i(\mK_\nu^\vx)}{\sigma^2}
            + \frac{\gamma_i(\mK_\mu^\vx)}{\gamma_i(\mK_\nu^\vx)}
        }
        \le n \parens*{
            \log \frac{\gamma_1(\mK_\nu^\vx)}{\sigma^2} + \frac{\gamma_1(\mK_\mu^\vx)}{\gamma_n(\mK_\nu^\vx)}
        }.
    \end{equation}
    Note that, by assumption, $\norm{\vm_\mu^\vx}_\infty, \norm{\vm_\nu^\vx}_\infty \le M$ and
    $\gamma_n(\mK^\vx_\mu),\gamma_n(\mK^\vx_\nu) \ge \sigma^2$.
    Moreover, it is true that $\gamma_1(\mK^\vx_\mu) \le \norm{\mK^\vx_\mu}$ for any matrix norm $\norm{\vardot}$.
    Taking this norm to be the $\infty$-norm $\norm{\mK^\vx_\mu}_\infty = \max_{i \in [n]} \sum_{j=1}^n \abs{(\mK^\vx_\mu)_{ij}}$, we see that $\gamma_1(\mK^\vx_\mu) \le n M$;
    similarly, $\gamma_1(\mK^\vx_\nu) \le n M$.
    Plugging in these estimates, we obtain
    \begin{equation}
        \text{(i)}
        \le n \parens*{
            \log \frac{n M}{\sigma^2} + \frac{n M}{\sigma^2}
        }
        \le \frac{2 n^2 M}{\sigma^2}.
    \end{equation}
    The bound for (ii) is simpler:
    \begin{equation}
        \text{(ii)}
        \le \frac{n}{\gamma_n(\mK_\nu^\vx)}\norm{\vm_\mu^\vx - \vm_\nu^\vx}^2_\infty
        \le \frac{4 n M^2}{\sigma^2}.
    \end{equation}
    Combining the bounds for (i) and (ii) gives the desired result.
\end{proof}

In the following, we will repeatedly make use of the following fact.
Let $\mu \in \overline{\mathcal{P}}$ have a mean function and covariance function and let $\nu \in \overline{\mathcal{P}}\ss{G}$.
Then, for all $\vx \in I$,
\begin{equation} \label{eq:G_fact}
    \G(P_\vx \mu, P_\vx \nu)
    = \KL(P_\vx \mu, P_\vx \nu) - \KL(P_\vx \mu, P_\vx \,\GP(\mu))
    = \KL(P_\vx\, \GP(\mu), P_\vx \nu)
    \ge 0
\end{equation}
with equality if and only if $P_\vx\, \GP(\mu) = P_\vx \nu$.
See \cref{app:Gaussian_divergence} for more details.

\begin{prop} \label{prop:minimiser_measure}
    Assume the assumptions of \cref{prop:bound}.
    Let $p(\vx)$ be a Borel distribution over $\X^n$ with full support.
    Then
    \begin{equation}
        \argmin_{\nu \in \overline{\mathcal{P}}\ss{G}}\, \E_{p(\vx)}[\KL(P_\vx \mu, P_\vx \nu)]
        = \argmin_{\nu \in \overline{\mathcal{P}}\ss{G}}\, \E_{p(\vx)}[\G(P_\vx \mu, P_\vx \nu)]
        = \GP(\mu).
    \end{equation}
\end{prop}
\begin{proof}
    By \cref{prop:bound}, $\vx \mapsto \KL(P_\vx \mu, P_\vx \, \GP(\mu))$ bounded.
    Therefore, we can decompose
    \begin{align}
        &\E_{p(\vx)}[\KL(P_\vx \mu, P_\vx \nu)] \nonumber \\
        &\quad = \E_{p(\vx)}[\KL(P_\vx \mu, P_\vx \nu) - \KL(P_\vx \mu, P_\vx\, \GP(\mu)) + \KL(P_\vx \mu, P_\vx\, \GP(\mu))] \\
        &\quad = \underbracket{
            \E_{p(\vx)}[\G(P_\vx \mu, P_\vx \nu)]
        }_{\text{(i)}}
        +
        \underbracket{
            \E_{p(\vx)}[\KL(P_\vx \mu, P_\vx\, \GP(\mu))]
        }_{\text{(ii)}},
    \end{align}
    using \eqref{eq:G_fact}.
    Here (i) measures how far $P_\vx \nu$ is from the best Gaussian approximation of $P_\vx \mu$ and (ii) measures the unavoidable approximation error due to the restriction to only Gaussian $P_\vx \nu$.
    In particular, (i) is zero if and only if $\G(P_\vx \mu, P_\vx \nu) = \KL(P_\vx\, \GP(\mu), P_\vx \nu) = 0$ for almost all $\vx \in \X^n$.
    Since, $\GP(\mu) \in \overline{\mathcal{P}}\ss{G}$ and $\nu \in \overline{\mathcal{P}}\ss{G}$, this is true if and only if $\GP(\mu) = \nu$ (\cref{prop:equality_noisy_processes}), which proves the result.
\end{proof}

\begin{prop} \label{prop:bound_over_data_sets}
    Let $f$ be a noisy process and let $\pi_f$ be the associated posterior prediction map.
    Let $\pi \in \overline{\M}\ss{G}\us{f.d.d.}$.
    Moreover, let $p(\vx)$ be a Borel distribution with full support over $\X^n$ for a fixed size $n \ge 2$,
    and let $p(D)$ be a Borel distribution with full support over a collection of data sets $\tilde \D \sub \D$ .
    Assume the following:
    \begin{enumerate}
        \item[(1)]
            The collection of data sets $\tilde \D$ is bounded (\cref{def:bounded_collection_of_data_sets}).
        \item[(2)]
            The process $f$ and prediction map $\pi$ have uniformly bounded second moments:
            \[
                \sup_{x \in \X} \E[f^2(x)]\, < \infty
                \quad\text{and}\quad
                \sup_{D \in \tilde D} \sup_{x \in \X}\, \E_{\pi(D)}[f^2(x)] < \infty.
            \]
            \vspace*{-1.5em}
        \item[(3)]
            The process $f$ is noisy with noise variance $\sigma^2 > 0$.
            Also, for all $D \in \tilde \D$, the process $\pi(D)$ is noisy with noise variance $\sigma_D^2 > 0$, and $\inf_{D \in \tilde \D} \sigma^2_D > 0$.
    \end{enumerate}
    Then
    \begin{equation} \label{eq:bound_over_data_sets}
        \E_{p(\vx)p(D)}[\KL(P_\vx \pi_f(D), P_\vx \pi(D))] < \infty.
    \end{equation}
\end{prop}
\begin{proof}
    To begin with, using \cref{prop:noisy_posterior_prediction_map_is_continuous}, we confirm that $\pi_f \in \overline{\M} \sub \overline{\M}\us{f.d.d.}$.
    To show \eqref{eq:bound_over_data_sets}, we show that the supremum over the bounds \eqref{eq:bound} is finite, which amounts to showing that (i) the data sets sizes are bounded, (ii) the collection of mean functions are covariance functions is uniformly bounded, and (iii) the collection of noise variances is bounded away from zero.
    These follow directly from respectively assumptions (1), (2) in combination with \cref{prop:posterior_prediction_map_bounded_second_moment} and (1), and (3).
\end{proof}

\begin{prop} \label{prop:minimiser_map}
    Assume the assumptions of \cref{prop:bound_over_data_sets}.
    Suppose that $\tilde \D$ is open.
    Then
    \begin{align}
        &\argmin_{\pi \in \overline{\mathcal{M}}\us{f.d.d.}\ss{G}}\, \E_{p(D)p(\vx)}[\KL(P_\vx \pi_f(D), P_\vx \pi(D))] \nonumber \\
        &\quad\overset{\tilde\D}{=} \argmin_{\pi \in \overline{\mathcal{M}}\us{f.d.d.}\ss{G}}\, \E_{p(D)p(\vx)}[\G(P_\vx \pi_f(D), P_\vx \pi(D))]
        \overset{\tilde\D}{=} \GP(\pi_f)
    \end{align}
    where the equalities hold for all $D \in \tilde \D$.
\end{prop}
\begin{proof}
    To begin with, using assumption (2) and \cref{prop:MM_is_continuous}, we confirm that $\GP(\pi_f) \in \overline{\mathcal{M}}\ss{G}\us{f.d.d.}$.
    Then
    \begin{equation}
        \E_{p(D)p(\vx)}[\KL(P_\vx \pi_f(D), P_\vx\, \GP(\pi_f)(D))] < \infty
    \end{equation}
    by \cref{prop:bound_over_data_sets}.
    Using this, decompose
    \begin{align} \label{eq:decomposition_objective_over_data_sets}
        &\E_{p(D)p(\vx)}[\KL(P_\vx \pi_f(D), P_\vx \pi(D))] \nonumber \\
        &\quad =
            \underbracket{
                \E_{p(D)p(\vx)}[\G(P_\vx \pi_f(D), P_\vx \pi(D))]
            }_{\text{(i)}}
            +
            \underbracket{
                \E_{p(D)p(\vx)}[\KL(P_\vx \pi_f(D), P_\vx\, \GP(\pi_f)(D))]
            }_{\text{(ii)}},
    \end{align}
    using \eqref{eq:G_fact}.
    Here (i) measures how far $P_\vx \pi(D)$ is from the best Gaussian approximatio of $P_\vx \pi_f(D)$ and (ii) measures the unavoidable approximation error due to the restriction to only Gaussian $P_\vx \pi(D)$.
    In particular, (i) is zero if and only if $\E_{p(\vx)}[\G(P_\vx \pi_f(D), P_\vx \pi(D))] = 0$ for almost all $D \in \tilde \D$.
    Consequently, by \cref{prop:minimiser_measure}, (i) is zero if and only if $\GP(\pi_f)(D) = \pi(D)$ for almost all $D \in \tilde \D$.
    Using that $\tilde \D$ is open and that $p(D)$ has full support, this set of probability one is dense.
    Therefore, since $\GP(\pi_f) \in \overline{\mathcal{M}}\ss{G}\us{f.d.d.}$ and $\pi \in \overline{\mathcal{M}}\ss{G}\us{f.d.d.}$,
    (i) is zero if and only if $\GP(\pi_f)(D) = \pi(D)$ for all $D \in \tilde \D$ (\cref{prop:equality_of_prediction_maps}), which proves the result.
\end{proof}

\begin{prop} \label{prop:minimising_sequence}
    Assume the assumptions of \cref{prop:bound_over_data_sets}.
    Let $(\pi_i)_{i \ge 1} \sub \overline{\M}\ss{G}\us{f.d.d.}$ be a minimising sequence for the infimum
    \begin{equation}
        \inf_{\pi \in \overline{\mathcal{M}}\us{f.d.d.}\ss{G}} \E_{p(D)p(\vx)}[\KL(P_\vx \pi_f(D), P_\vx \pi(D))].
    \end{equation}
    \begin{proplist}
        \item \label{prop:limit_is_right}
            Suppose that, for almost all $D \in \tilde D$, $\pi_i(D)$ has a weak limit $\pi^*(D) \in \overline{\mathcal{P}}\ss{G}$.
            Then $\pi^*(D) = \GP(\pi_f)(D)$ for almost all $D \in \tilde \D$.
        \item \label{prop:limit_along_subsequence}
            Suppose that, for almost all $D \in \tilde \D$, $\pi_i(D)$ satisfies the following property:
            if there exist some $\mu \in \mathcal{P}$ and dense $\tilde I \sub I$ such that
            $P_\vx \pi_i(D) \weakto P_\vx \mu$ for all $\vx \in \tilde I$, then $\pi_i(D) \weakto \mu$.
            Then there exists a subsequence $(\pi_{n_i})_{i \ge 1}$ of $(\pi_i)_{i \ge 1}$ such that $\pi_{n_i}(D) \weakto \GP(\pi_f)(D)$ for almost all $D \in \tilde D$.
    \end{proplist}
\end{prop}
\begin{proof}
    To begin with, by \cref{prop:bound_over_data_sets}, $\E_{p(D)p(\vx)}[\KL(P_\vx \pi_f(D), P_\vx\, \GP(\pi_f)(D))] < \infty$.
    Therefore, decompose
    \begin{align}
        &\E_{p(D)p(\vx)}[\KL(P_\vx \pi_f(D), P_\vx \pi_i(D))]  \nonumber \\
        \label{eq:decomposition_for_convergence}
        &\quad =
            \E_{p(D)p(\vx)}[\G(P_\vx \pi_f(D), P_\vx \pi_i(D))]
            +  \E_{p(D)p(\vx)}[\KL(P_\vx \pi_f(D), P_\vx\, \GP(\pi_f)(D))]
    \end{align}
    using \eqref{eq:G_fact}.
    This shows that
    \begin{equation} \label{eq:KL_implies_G_converges}
        \lim_{i \to \infty} \E_{p(D)p(\vx)}[\G(P_\vx \pi_f(D), P_\vx \pi_i(D))] = 0
    \end{equation}
    because $\E_{p(D)p(\vx)}[\G(P_\vx \pi_f(D), P_\vx\, \GP(\pi_f)(D))] = 0$ and $\GP(\pi_f) \in \overline{\M}\ss{G}\us{f.d.d.}$ (\eqref{eq:G_fact} and \cref{prop:bound_over_data_sets}).

    \cref{prop:limit_is_right}:
    By Fatou's Lemma and the fact that $(\mu, \nu) \mapsto \KL(\mu, \nu)$ is weakly lower semi-continuous \citep{Posner:1975:Random_Coding_Strategies_for_Minimum},
    \begin{equation}
        \liminf_{i \to \infty}\, \E_{p(D)p(\vx)}[\KL(P_\vx \pi_f(D), P_\vx \pi_i(D))]
        \ge \E_{p(D)p(\vx)}[\KL(P_\vx \pi_f(D), P_\vx \pi^*(D))].
    \end{equation}
    Therefore, by \eqref{eq:decomposition_for_convergence},
    \begin{equation}
        \liminf_{i \to \infty}\, \E_{p(D)p(\vx)}[\G(P_\vx \pi_f(D), P_\vx \pi_i(D))]
        \ge \E_{p(D)p(\vx)}[\G(P_\vx \pi_f(D), P_\vx \pi^*(D))].
    \end{equation}
    But the left-hand side is zero by \eqref{eq:KL_implies_G_converges}, so the right-hand side is also zero:
    \begin{equation}
        \E_{p(D)p(\vx)}[\G(P_\vx \pi_f(D), P_\vx \pi^*(D))] = 0,
    \end{equation}
    which yields that $\E_{p(\vx)}[\G(P_\vx \pi_f(D), P_\vx \pi^*(D))] = 0$ for almost all $D \in \tilde \D$.
    Consequently, using that $\pi^*(D) \in \overline{\mathcal{P}}\ss{G}$, it follows that $\pi^*(D) = \GP(\pi_f)(D)$ for almost all $D \in \tilde \D$ (\cref{prop:minimiser_measure}).

    \cref{prop:limit_along_subsequence}:
    Note that
    \begin{equation}
        \lim_{i \to \infty} \E_{p(D)p(\vx)}[\G(P_\vx \pi_f(D), P_\vx \pi_i(D))] = 0.
    \end{equation}
    Therefore, there exists a collection of data sets $A \sub \tilde \D$ of probability one such that, along a subsequence,
    \begin{equation}
        \lim_{i \to \infty} \E_{p(\vx)}[\G(P_\vx \pi_f(D), P_\vx \pi_i(D))] = 0 \quad\text{for all $D \in A$}.
    \end{equation}
    We show that $\pi_i(D) \weakto \GP(\pi_f)(D)$ for all $D \in A$.
    Let $D \in A$.
    Pass to a further subsequence of $(\pi_i(D))_{i \ge 1}$.
    It suffices to show that $(\pi_i(D))_{i \ge 1}$ contains a another further subsequence weakly convergent to $\GP(\pi_f)(D)$.
    Start with the observation that it still holds that
    \begin{equation}
        \lim_{i \to \infty} \E_{p(\vx)}[\G(P_\vx \pi_f(D), P_\vx \pi_i(D))] = 0.
    \end{equation}
    Hence, there exists a collection of index sets $B \sub \X^n$ of probability one such that, along a further subsequence,
    \begin{equation}
        \lim_{i \to \infty} \G(P_\vx \pi_f(D), P_\vx \pi_i(D)) = 0 \quad \text{for all $\vx \in B$}.
    \end{equation}
    Consequently, by Pinsker's Inequality and \eqref{eq:G_fact}, $P_\vx \pi_i(D) \weakto \Normal(P_\vx \pi_f(D)) = P_\vx\, \GP(\pi_f)(D)$ for all $\vx \in B$.
    In particular, using that $n \ge 2$, this means that almost all means and covariances converge, which in turn means that $P_\vx \pi_i(D) \weakto P_\vx\, \GP(\pi_f)(D)$ for all $\vx \in \tilde I$, for some dense $\tilde I \sub I$.
    Therefore, by the assumed property, we conclude that $\pi_i(D) \weakto \GP(\pi_f)(D)$.
\end{proof}

\clearpage
\section{The Gaussian Neural Process}
\label{app:GNP}

We build on the development from \cref{sec:objective-function-analysis,sec:GNP}, where we defined a Gaussian approximation $\tilde \pi \colon \D \to \mathcal{P}\ss{G}$ of the posterior prediction map $\pi_f\colon \D \to \mathcal{P}$ corresponding some ground truth stationary stochastic process $f$.
Recall that stationarity of $f$ is equivalent to translation equivariance of $\pi_f$ \citep{Foong:2020:Meta-Learning_Stationary_Stochastic_Process_Prediction}:
for all $D \in \D$ and $\tau \in \X$,
\begin{equation}
    \T_\tau \pi_f(D) = \pi_f(D + \tau)
\end{equation}
where $\T_\tau f = f(\vardot - \tau)$ is the \emph{shifting operator}, $\T_\tau  \pi_f(D)$ is the measure $\pi_f(D)$ pushed through $\T_\tau$, and
$D + \tau = (\vx, \vy) + \tau = ((x_1 + \tau, \ldots, x_{\abs{\vx}} + \tau), \vy)$.

Since the approximation $\tilde \pi$ is Gaussian, in a way that we now make precise, translation equivariance of $\tilde \pi$ is characterised by translation equivariance of the mean functions and kernel functions that $\tilde \pi$ maps to.
For all $D \in \D$, denote $\tilde \pi(D) = \GP(m(D)(\vardot), k(D)(\vardot, \vardot))$. \sloppy
Then $\tilde \pi$ is translation equivariant if and only if, for all $D \in \D$ and $\tau \in \X$,
\begin{equation} \label{eq:TE_mean_kernel}
    m(D + \tau)(\vardot) = m(D)(\vardot - \tau) \quad \text{and} \quad
    k(D + \tau)(\vardot, \vardot) = k(D)(\vardot - \tau, \vardot - \tau).
\end{equation}
We proceed to find general parametrisations of the mean mapping $m$ and kernel mapping $k$ that then provide a general parametrisation of our approximation $\tilde \pi$.

Consider an arbitrary mean mapping $m \colon \D \to C(\X, \Y)$ and kernel mapping $k \colon \D \to C\us{p.s.d.}(\X^2, \Y)$ where $C\us{p.s.d.}(\X^2, \Y)$ is the collection of continuous positive semi-definite functions $\X^2 \to \Y$.
Suppose that these mappings satisfy \cref{eq:TE_mean_kernel}, \ie~they are translation equivariant.
Assume that $m$ and $k$ are continuous with respect to the metric on $\D$ (\cref{app:notation}) and compact convergence on $C(\X, \Y)$ and $C\us{p.s.d.}(\X^2, \Y)$.

The goal of this appendix is twofold:
establish a universal representation for the kernel mapping $k$ (\cref{subsec:kernel_representation}) and an implementable neural architecture that can approximate this representation (\cref{subsec:kernel_architecture}).
Moreover, in \cref{subsec:training_objective}, we formulate an objective that can be used to train the parameters of this architecture.

\subsection{Universal Representation of the Kernel Map}
\label{subsec:kernel_representation}

Before we turn our attention to the kernel mapping $k$, we review how Thm 1 by \citet{Gordon:2020:Convolutional_Conditional_Neural_Processes} can be used to establish a universal representation of the mean mapping $m$:
for a collection of data sets $\tilde \D \sub \D$ that is topologically closed, closed under permutations, and closed under translations with finite maximum data set size and multiplicity $K \in \N$ \citep[Def 2 by][]{Gordon:2020:Convolutional_Conditional_Neural_Processes}---intuitively, the number of times an observation can occur at the same input is at most $K$---there exists a Hilbert space $\H$ of functions on $\X$, a continuous stationary kernel $\psi\colon \X \to \R$, a continuous $\phi\colon\Y\to\R^{K+1}$, and a continuous and translation-equivariant $\rho \colon \H' \to C(\X,\Y)$ such that, for all $D \in \tilde \D$,
\begin{equation}
    m(D) = \rho(E(D))
    \quad \text{with} \quad
    E(\vx, \vy) = \sum_{i=1}^{\abs{\vx}} \phi(y_i) \psi(\vardot - x_i),
\end{equation}
where $\H' = E(\tilde \D) \sub \H$ is a closed subset of $\H$.

We find a similar representation for the kernel mapping $k$ by reducing it to a case where Thm 1 by \citet{Gordon:2020:Convolutional_Conditional_Neural_Processes} can be applied.
Consider a data set $D = (\vx, \vy) \in \D$.
Embed the data set in $\mathcal{D}_2 := \union_{n=0}^\infty (\X^2\times \Y)^n$ by duplicating the inputs:
\begin{equation} \label{eq:duplication}
    D'
    = (((x_1, x_1), \ldots, (x_{\abs{\vx}}, x_{\abs{\vx}})), \vy)
    =: (\operatorname{duplicate}(\vx), \vy).
\end{equation}
Then $k$ satisfies
\begin{equation}
    k(D' + (\tau, \tau))(\vardot,\vardot) = k(D')(\vardot - \tau, \vardot - \tau)
    \quad \text{for all $(\tau, \tau) \in \X^2$}.
\end{equation}
In other words, $k$ can be viewed as a continuous function $\D_2 \to C(\X^2, \Y)$ that is equivariant with respect to \emph{diagonal translations}.
If we can continuously extend $k$ to be equivariant with respect to \emph{all} translations, then are in a position to apply Thm 1 by \citet{Gordon:2020:Convolutional_Conditional_Neural_Processes}, now for the input space $\X^2$.

We provide an explicit construction of this desired continuous extension.
Set $\ve_\parallel = (1, 1) / \sqrt{2} \in \X^2$ and $\ve_\perp = (1, -1) / \sqrt{2} \in \X^2$.
Then $\ve_\parallel$ and $\ve_\perp$ form an orthogonal basis for $\X^2$.
For $\vtau \in \X^2$, let $\T_\vtau f = f(\vardot - \vtau)$ be the shifting operator operating on functions on $\X^2$.
Lift $k \colon \D \to C(\X, \Y)$ to $\overline k\colon \D_2 \setminus \set{\es} \to C(\X^2, \Y)$ by setting
\begin{equation}
    \overline k((\vx_1, \ldots, \vx_n), \vy) = \begin{cases}
        k((x_{11}, \ldots, x_{n1}), \vy) & \text{if $x_{i1} = x_{i2}$ for all $i \in [n]$,} \\
        0 & \text{otherwise}.
    \end{cases}
\end{equation}
Note that $D = \es$ is excluded;
we will turn to this case after \cref{lem:extension-kernel-mapping}.
Then
\begin{equation}
    \overline k((\vx_1 - \vtau, \ldots, \vx_n - \vtau), \vy)
    = \T_{\vtau} \overline k((\vx_1, \ldots, \vx_n), \vy)
    \quad \text{for all $\vtau = (\tau, \tau) \in \X^2$},
\end{equation}
which is the earlier established property that $k$ is equivariant with respect to diagonal translations.
Finally, we construct the desired extension $\hat k \colon \D_2 \setminus \set{\es} \to C(\X^2, \Y)$:
\begin{equation}
    \hat k((\vx_1, \ldots, \vx_n), \vy)
    = \T_{\lra{\ve_\perp, \vx\ss{c}}\ve_\perp} \overline k((\lra{\ve_\parallel, \vx_1}\ve_\parallel, \ldots, \lra{\ve_\parallel, \vx_n}\ve_\parallel), \vy)
\end{equation}
where $\vx\ss{c} = \frac1n\sum_{i=1}^n \vx_i$.

\begin{lem} \label{lem:extension-kernel-mapping}
    The extended kernel mapping $\hat k \colon \D_2 \setminus \set{\es} \to C(\X^2, \Y)$
    \begin{lemlist}
        \item \label{lem:extension}
            extends $k$: it agrees with $k$ on the embedding of $\D \setminus\set{\es}$ in $\D_2$;
        \item \label{lem:TE}
            is permutation invariant and translation equivariant; and
        \item \label{lem:continuity}
            is continuous with respect to the metric on $\D$ (\cref{app:notation}) and compact convergence on $C(\X^2, \Y)$.
    \end{lemlist}
\end{lem}
\begin{proof}
    \cref{lem:extension}:
    If all inputs $\vx = (x, x) \in \X^2$, then $\lra{\ve_\parallel, \vx_1}\ve_\parallel = \vx$ and $\lra{\ve_\perp, \vx_1}\ve_\perp = \vnull$, so it is clear that $\hat k$ then agrees with $k$.

    \cref{lem:TE}:
    That $\hat k$ is permutation invariant is clear.
    We check translation equivariance.
    Let $\vtau \in \X^2$.
    Then $\vtau = \vtau_\parallel + \vtau_\perp$ where  $\vtau_\parallel = \lra{\ve_\parallel, \vtau}\ve_\parallel$ and $\vtau_\perp = \lra{\ve_\perp, \vtau}\ve_\perp$.
    Therefore,
    \begin{align}
        &\hat k((\vx_1 + \vtau, \ldots, \vx_n + \vtau), \vy) \\
        &\quad = \T_{\lra{\ve_\perp, \vx\ss{c} + \vtau}\ve_\perp} \overline k((\lra{\ve_\parallel, \vx_1 + \vtau}\ve_\parallel, \ldots, \lra{\ve_\parallel, \vx_n + \vtau}\ve_\parallel), \vy) \\
        &\quad = \T_{\lra{\ve_\perp, \vx\ss{c}}\ve_\perp + \vtau_\perp} \overline k((\lra{\ve_\parallel, \vx_1}\ve_\parallel + \vtau_\parallel, \ldots, \lra{\ve_\parallel, \vx_n}\ve_\parallel + \vtau_\parallel), \vy) \\
        &\quad = \T_{\lra{\ve_\perp, \vx\ss{c}}\ve_\perp + \vtau_\perp} \T_{\vtau_\parallel} \overline k((\lra{\ve_\parallel, \vx_1}\ve_\parallel, \ldots, \lra{\ve_\parallel, \vx_n}\ve_\parallel), \vy) \\
        &\quad = \T_{\vtau} \T_{\lra{\ve_\perp, \vx\ss{c}}\ve_\perp} \overline k((\lra{\ve_\parallel, \vx_1}\ve_\parallel, \ldots, \lra{\ve_\parallel, \vx_n}\ve_\parallel), \vy) \\
        &\quad = \T_{\vtau} \hat k((\vx_1, \ldots, \vx_n), \vy).
    \end{align}
    Since $\vtau \in \X^2$ was arbitrary, this shows that $\hat k$ is translation equivariant.

    \cref{lem:continuity}:
    For $i \in [n]$, let $(\vx_i^{(l)})_{l \ge 1} \sub \X^2$ be convergent to $\vx_i \in \X^2$, and let $(\vy_l)_{l \ge 1} \sub \Y^n$ be convergent to $\vy \in \Y^n$.
    Set $\vtau_l = \lra{\ve_\perp, \vx^{(l)}\ss{c}} \ve_\perp$, $\vtau = \lra{\ve_\perp, \vx\ss{c}} \ve_\perp$,
    \begin{equation}
        f_l = \overline k((\lra{\ve_\parallel, \vx^{(l)}_1}\ve_\parallel, \ldots, \lra{\ve_\parallel, \vx^{(l)}_n}\ve_\parallel), \vy_l),
        \quad
        f = \overline k((\lra{\ve_\parallel, \vx_1}\ve_\parallel, \ldots, \lra{\ve_\parallel, \vx_n}\ve_\parallel), \vy).
    \end{equation}
    Then $\vtau_l \to \vtau$ and, by continuity of $k$, $f_l \to f$ compactly.
    Hence, it remains to show that $\T_{\vtau_l} f_l \to \T_{\vtau} f$ compactly.
    By convergence of $(\vtau_l)_{l \ge 1}$, assume that the sequence $(\vtau_l)_{l \ge 1}$ and limit $\vtau$ are contained in $[-R, R]^2$ for some $R > 0$.
    Let $M > 0$ and consider $\vx \in [-M, M]^2$.
    Set $K = [-(R + M), (M + R)]^2$.
    Estimate
    \begin{align}
        \abs{
            \T_{\vtau_l} f_l(\vx) - \T_{\vtau} f(\vx)
        }
        &\le
            \abs{
                f_l(\vx - \vtau_l) - f(\vx - \vtau_l)
            }
            +
            \abs{
                f(\vx - \vtau_l) - f(\vx - \vtau)
            } \\
        &\le
            \underbracket{\sup_{\vz \in K \vphantom{K^2\norm{\vtau_l}}}\, \abs{f_l(\vz) - f(\vz)}}_{\text{(i)}}
            +
            \underbracket{
                \sup_{
                    \vz, \vz' \in K^2, \,
                    \norm{\vz - \vz'}_2 \le \norm{\vtau - \vtau_l}_2
                } \abs{f(\vz) - f(\vz')}
            }_{\text{(ii)}}.
    \end{align}
    Here (i) $\to 0$ because $f_l \to f$ compactly and (ii) $\to0$ because $f$ is continuous on $K$ hence uniformly continuous on $K$ ($K$ is compact).
    We conclude that $\T_{\vtau_l} f_l \to \T_{\vtau} f$ compactly.
\end{proof}

For $D = \es$, we simply set $\hat k(\es) = k(\es)$.
Note that there are no issues of continuity of $\hat k$ at $\es$, because $\es$ is an isolated point of $\tilde \D$ (\cref{app:notation}).

Let $\tilde \D \sub \D$ be collection of data sets that is topologically closed, closed under permutations, and closed under translations with finite maximum data set size and multiplicity $K \in \N$.
Let $\tilde \D_2$ be $\tilde \D$ embedded in $\D_2$ by duplicating the inputs and allowing for a translation:
\begin{equation}
    \tilde \D_2
    = \set{(\operatorname{duplicate}(\vx), \vy) +  \vtau, :
    (\vx, \vy) \in \tilde \D, \,
    \vtau \in \X^2
    }.
\end{equation}
Then also $\tilde \D_2$ is topologically closed, closed under permutations, and closed under translations, has finite maximum data set size, and has multiplicity $K$.
Following Thm 1 by \citet{Gordon:2020:Convolutional_Conditional_Neural_Processes}, let the encoding of a data set in $\D_2$ be
\begin{equation}
    E_2\colon \tilde \D_2 \to E_2(\tilde\D_2), \quad
    E((\vx_1, \ldots, \vx_n), \vy) = \sum_{i=1}^{n} \phi(y_i) \psi(\vardot - \vx_i),
\end{equation}
where we abuse notation to immediately restrict $E_2$ to its image.
According to Lems 1 to 4 by \citet{Gordon:2020:Convolutional_Conditional_Neural_Processes}, $\H' = E(\tilde \D_2) \sub \H$ is a closed subset of $\H$, and $E$ is a translation-equivariant homeomorphism where the inverse recovers the input data set up to a permutation.
Set $\rho = \hat k \comp E_2^{-1} \colon \H' \to C(\X^2, \Y)$.
Then, by \cref{lem:extension},
\begin{equation}
    \rho(E(\operatorname{duplicate}(\vx), \vy)) = k(D)
    \text{ for all $D=(\vx, \vy) \in \tilde \D$;}
\end{equation}
and, by \cref{lem:continuity}, $\rho$ is continuous.
Moreover, by \cref{lem:TE}, $\rho$ is translation equivariant on $E_2(\tilde \D_2 \setminus \set{\es})$.
The construction breaks down with translation equivariance of $\rho$ at the zero function $E_2(\es) = 0$.
We discuss this issue next.

Suppose that $\rho$ were also translation equivariant at the zero function $E_2(\es) = 0$.
Then $\rho(0) = \rho(\T_\vtau 0) = \T_\vtau \rho(0)$ for all $\vtau \in \X^2$, which means that $\rho(E_2(\es))$ must be a constant function.
This is an issue, because $k(\es)$ is not a constant function.
We fix the issue by avoiding the zero function entirely.
In particular, we increase the dimensionality of the embedding $E_2$ by one by concatenating some fixed continuous function $h \in C(\X^2, \Y)$:
\begin{equation}
    E_2\colon \tilde\D_2 \to E_2(\tilde \D_2), \quad
    E((\vx_1, \ldots, \vx_n), \vy) = \begin{bmatrix}
        \displaystyle\sum_{i=1}^{n} \phi(y_i) \psi(\vardot - \vx_i) \\
        \T_{\vx\ss{c}} h(\vardot)
    \end{bmatrix}
\end{equation}
where we again abuse notation to immediately restrict $E_2$ to its image and set
$\vx\ss{c} = \frac1n \sum_{i=1}^n \vx_i$ if $n > 0$ and $\vx\ss{c} = \vnull$ otherwise.
Then clearly $E_2(\tilde \D_2) \sub \H \times C(\X^2, \Y)$ is still a closed subset of $\H \times C(\X^2, \Y)$ and clearly $E_2$ is still a translation-equivariant homeomorphism.
Again, set $\rho = \hat k \comp E_2^{-1}$.
Then again $\rho$ is continuous and agrees with $k$.
The key difference is that $E_2(\es)$ is now not equal to the zero function, so $\rho$ can be extended to also be translation equivariant at $E_2(\es)$:
set $\rho(\T_\vtau E_2(\es)) := \T_\vtau \rho(E(\es)) = \T_\vtau k(\es)$ for all $\vtau \in \X^2$.
We need to make sure that this extension of $\rho$ well defined.
For a translation $\vtau \in \X^2$,
denote $\vtau_\parallel = \lra{\ve_\parallel, \vtau}\ve_\parallel$ and $\vtau_\perp = \lra{\ve_\perp, \vtau}\ve_\perp$.

\begin{dfn}
    Let $h \in C(\X^2, \Y)$.
    Call $h$ \emph{$\ve_\perp$-discriminating} if, for all translations $\vtau_1 \in \X^2$ and $\vtau_2 \in \X^2$, we have $\T_{\vtau_1} h \neq \T_{\vtau_2} h$ whenever $(\vtau_{1})_\perp \neq (\vtau_{2})_{\perp}$.
\end{dfn}

\begin{lem} \label{lem:extension_rho}
    Suppose that $h$ is $\ve_\perp$-discriminating.
    If $\vtau_1 \in \X^2$ and $\vtau_2 \in \X^2$ are two translations such that $\T_{\vtau_1} E(\es) = \T_{\vtau_2} E(\es)$, then $\T_{\vtau_1} \rho(E(\es)) = \T_{\vtau_2} \rho(E(\es))$.
    In other words, the extension of $\rho$ is well defined.
\end{lem}
\begin{proof}
    Since $\T_{\vtau_1} E(\es) = \T_{\vtau_2} E(\es)$, in particular $\T_{\vtau_1} h = \T_{\vtau_2} h$.
    Therefore, using that $h$ is $\ve_\perp$-discriminating, $(\vtau_{1})_\perp = (\vtau_{2})_{\perp}$.
    Then
    \begin{equation}
        \T_{\vtau_1} \rho(E(\es))
        = \T_{\vtau_1} k(\es)
        \overset{\text{(i)}}{=}\T_{(\vtau_1)_\perp} k(\es)
        \overset{\text{(ii)}}{=} \T_{(\vtau_2)_\perp} k(\es)
        \overset{\text{(i)}}{=} \T_{\vtau_2} k(\es)
        = \T_{\vtau_2} \rho(E(\es)),
    \end{equation}
    using in (i) that $k(\es)$ is invariant to diagonal translations
    and in (ii) that $(\vtau_{1})_\perp = (\vtau_{2})_{\perp}$.
\end{proof}

We make one last simplifying assumption:
let $h$ be \emph{invariant} to diagonal translations.
Then
\begin{equation}
    E_2(\operatorname{duplicate}(\vx), \vy) = \begin{bmatrix}
        \displaystyle\sum_{i=1}^{n} \phi(y_i) \psi(\vardot - (x_i, x_i)) \\
        h(\vardot)
    \end{bmatrix}
\end{equation}
where $\set{E_2(\operatorname{duplicate}(\vx), \vy) : (\vx, \vy) \in \tilde \D} = \H' \times \set{h}$ with $\H' \sub \H$ a closed subset of $\H$.
We have proved the following theorem.

\begin{thm} \label{thm:kernel_representation}
    Let $k \colon \D \to C\us{p.s.d.}(\X^2, \Y)$ be a continuous and translation-equivariant kernel mapping.
    Let $\tilde \D \sub \D$ be collection of data sets that is topologically closed, closed under permutations, and closed under translations with finite maximum data set size and multiplicity $K \in \N$.
    Set $\phi\colon\Y\to\R^{K+1}$, $\phi(y) = (y^0, y^1, \ldots, y^{K})$.
    Choose any $h \in C(\X^2, \Y)$ that is $\ve_\perp$-discriminating and invariant with respect to diagonal translations.
    Then there exists a reproducing kernel Hilbert space $\H$ of functions on $\X^2$, a continuous stationary kernel $\psi\colon \X \to \R$, and a continuous and translation-equivariant $\rho \colon \H' \to C(\X^2,\Y)$
    such that, for all $D \in \tilde \D$,
    \begin{equation}
        k(D) = \rho(E(D))
        \quad \text{with} \quad
        E(\vx, \vy) = \begin{bmatrix}
            \displaystyle\sum_{i=1}^{n} \phi(y_i) \psi(\vardot - (x_i, x_i)) \\
            h(\vardot)
        \end{bmatrix}
    \end{equation}
    where $\H' = E_2(\tilde \D_2) \sub \H \times C(\X^2, \Y)$ is a closed subset of $\H \times C(\X^2, \Y)$.
\end{thm}

We point out is that $\rho$ is only defined on the closed subset $\H'$.
Using a generalisation of Tietze's Extension Theorem by \citet{Dugundji:1951:An_Extension_of_Tietzes_Theorem}, $\rho$ can perhaps be continuously extended to the entirety of $\H$ and an appropriate space containing $h$, and it appears possible to perform this extension whilst preserving translation equivariance, see \eg~the thesis by \citet{Feragen:2006:Characterization_of_Equivariant_ANEs}.
We leave these investigations for future work.

\subsection{Implementation of the Kernel Map Representation}
\label{subsec:kernel_architecture}

In this section, we establish an implementable neural architecture that can approximate the representation in \cref{thm:kernel_representation}.
The key observation is that $\rho$ is a translation-equivariant map between two function spaces.
Therefore, if we discretise the functions finely enough, then it appears plausible that the resulting map between discretisations can be approximated by a CNN.
We will not make this statement precise;
rather, we point the reader to \citet{Yarotsky:2018:Universal_Approximations_of_Invariant_Maps} for the universal approximation properties of CNNs.

\newcommand{\CNN}{\operatorname{CNN}}
We formalise our approximating architecture.
We start out by approximating $\rho$.
Let $\mZ \in \R^{M \times M}$ be a sufficiently fine grid on $\X^2$.
Then, assuming a universal approximation capability of CNNs convenient for our purpose, let $\operatorname{CNN}$ be such that
\begin{equation}
    \norm{\rho(e)(\mZ) - \operatorname{CNN}(e(\mZ))}\ss{F} < \e\quad\text{for all $e \in E(\tilde \D)$,}
\end{equation}
for some chosen level of approximation accuracy $\e > 0$.
Although $\CNN(e(\mZ))$ approximates $\rho(e)(\mZ)$ satisfactorily, it is not guaranteed that $\CNN(e(\mZ))$ is a positive semi-definite or even symmetric matrix, which is required for our applications.
To fix this, let $\Pi\us{p.s.d.}$ be the operator that takes a matrix to the positive semi-definite matrix closest in Frobenius norm \citep{Higham:1988:Nearest_PSD}.
Let $e \in E(\tilde \D)$.
Then
\begin{align}
    &\norm{\rho(e)(\mZ) - \Pi\us{p.s.d.}\operatorname{CNN}(e(\mZ))}\ss{F} \\
    &\quad\le
        \norm{\rho(e)(\mZ) - \CNN(e(\mZ))}\ss{F}
        + \norm{\CNN(e(\mZ)) - \Pi\us{p.s.d.}\operatorname{CNN}(e(\mZ))}\ss{F}. \nonumber
\end{align}
The first term is less than $\e$ by choice of $\CNN$, and it is readily seen that the second term is also less than $\e$:
\begin{equation}
    \norm{\CNN(e(\mZ)) - \Pi\us{p.s.d.}\operatorname{CNN}(e(\mZ))}\ss{F}
    \le \norm{\CNN(e(\mZ)) - \rho(e)(\mZ)}\ss{F}
    < \e
\end{equation}
where ``$\le$'' follows from the definition of $\Pi\us{p.s.d.}$ and the observation that $\rho(e)(\mZ)$ is positive semi-definite.
Therefore,
\begin{equation}
    \norm{\rho(e)(\mZ) - \Pi\us{p.s.d.}\operatorname{CNN}(e(\mZ))}\ss{F} < 2\e\quad\text{for all $e \in E(\tilde \D)$,}
\end{equation}
which means that $\Pi\us{p.s.d.}\operatorname{CNN}$ is also a good approximation of $\rho$.
Crucially, the output $\Pi\us{p.s.d.}\operatorname{CNN}$ is always positive semi-definite, suited for our applications.
To complete the architecture, we assume that the discretisation $\mZ$ is sufficiently fine and far-reaching so that we can reasonably construct the covariance between any two points $x_1 \in \X$ and $x_2 \in \X$ of interest simply through interpolation:
\begin{align}
    k(D)(x_1, x_2)
    &\approx \sum_{i=1}^{M} \sum_{j=1}^M \hat\psi(x_1 - z_1) \hat\psi(x_2 - z_2) [\Pi\us{p.s.d.} \CNN(E(D)(\mZ))]_{ij} \\
    &= \lra{\hat \vpsi(x_1), \Pi\us{p.s.d.} \CNN(E(D)(\mZ), h(\mZ)) \hat \vpsi(x_2)} \label{eq:kernel_psd}
\end{align}
where $Z_{ij} = (z_i, z_j)$, $\hat\psi$ is some suitable interpolation kernel, and
\begin{equation}
    \hat \vpsi(x) = (\hat\psi(x - z_1), \ldots, \hat\psi(x - z_M)) \in \R^M
    \quad\text{for all $x \in \X$}.
\end{equation}
From \eqref{eq:kernel_psd} it is clear that the approximation of $k(D)(x_1, x_2)$ is indeed a positive semi-definite function.

Assume that the multiplicity $K$ of the collection of data sets is one: $K = 1$.
Moreover, let $h \in \H$ be such that that $h(\mZ) = \mI$;
indeed, this choice for $h \in \H$ is $\ve_\perp$-discriminating and invariant with respect to diagonal translations.
Let $D\us{(c)} = (\vx\us{(c)}, \vy\us{(c)}) \in \D$ be a data set and let $\vx\us{(t)} \in \X^{\abs{\vx\us{(t)}}}$ be some target inputs.
We then summarise the architecture in the following three-step procedure:
\begin{enumerate}
    \item[\circled{1}]
        Run the encoder $\mH = \operatorname{enc}(D\us{(c)}, \mZ)$, which is defined to produce the following three channels:
        \begin{equation}
            \hspace*{-10pt}
            \underbracket{
                \mH_{::1}
                = \sum_{i=1}^{\abs{\vx}} y\us{(c)}_i \psi(\mZ - (x\us{(c)}_i, x\us{(c)}_i)),
            }_{\text{data channel}}
            \;\;
            \underbracket{
                \mH_{::2} = \sum_{i=1}^{\abs{\vx}} \psi(\mZ - (x\us{(c)}_i, x\us{(c)}_i)),
            }_{\text{density channel}}
            \;\;
            \underbracket{
                \vphantom{\sum_{i=1}^{\abs{\vx}}}
                \mH_{::3} = \mI.
            }_{\mathclap{\text{source channel}}}
        \end{equation}
    \item[\circled{2}]
        Pass the encoding through a CNN, which outputs a single channel, and map it to the closest positive semi-definite matrix:
        \begin{equation}
            \mK = \Pi\us{p.s.d.} \CNN(\mH).
        \end{equation}
    \item[\circled{3}]
        Run the decoder to interpolate the covariances to the target inputs:
        \begin{equation}
            k(D)(\vx\us{(t)}, \vx\us{(t)})
            \approx
                \operatorname{dec}(\mK, \vx\us{(t)})
                := \psi'(\vx\us{(t)}, \vz^\T) \mK \psi'(\vz, \vx\us{(t)\T}).
        \end{equation}
\end{enumerate}

The three channels in \circled{1} all serve a distinct but crucial service:
The \emph{data channel} communicates the values of the observed data to the model.
However, as pointed out by \citet{Gordon:2020:Convolutional_Conditional_Neural_Processes}, the model will not be able to distinguish between a value $y_i = 0$ and no observation.
This is what the \emph{density channels} fixes:
it communicates to the model where the data is present and where data is missing.
Finally, the \emph{source channel} allows the architecture to ``start out'' with a stationary prior with covariance $\mI$, which corresponds to white noise, and then pass it through a CNN to modulate this prior to introduce correlations inferred from the context set.
\emph{C.f.}, any Gaussian process can be sampled from by first sampling white noise and then convolving this noise with an appropriate filter;
the kernel architecture is a nonlinear generalisation of this procedure.

\subsection{Training Objective}
\label{subsec:training_objective}
The \textsc{ConvDeepSet} for the mean mapping $m$ and the architecture described in \cref{subsec:kernel_architecture} for the kernel mapping $k$ form the model that we call the Gaussian Neural Process (\textsc{GNP}).
The \textsc{GNP} depends on some parameters $\vth$, \eg~weights and biases for the CNNs.
To train these parameters $\vth$, we maximise the objective \eqref{eq:objective_MC} from \cref{sec:objective-function-analysis}:
\begin{equation}
    \vth^*
    = \argmax_{\vth} \frac{1}{N} \sum_{i=1}^N \log \Normal(\vy\us{(t)}_i\cond \vm_\vth(D\us{(c)}_i, \vx\us{(t)}_i), \mK_\vth(D\us{(c)}_i, \vx\us{(t)}_i)),
\end{equation}
where $\vm_\vth(D\us{(c)}_i, \vx\us{(t)}_i)$ and $\mK_\vth(D\us{(c)}_i, \vx\us{(t)}_i)$ implement the \textsc{ConvDeepSet} architecture and kernel architecture from \cref{subsec:kernel_architecture} respectively, producing, from the context set $D\us{(c)}_i$, a mean and covariance matrix for the target inputs $\vx\us{(t)}_i$.
In practice, we use minibatching in combination with stochastic gradient descent and the adaptive step size method \textsc{ADAM} \citep{Kingma:2014:Adam_A_Method_for_Stochastic}.

\clearpage
\section{Experimental Setup}
\label{app:experiments}

We follow the experimental setup of \citet{Foong:2020:Meta-Learning_Stationary_Stochastic_Process_Prediction}, with the following exceptions:
\begin{enumerate}[label=(\arabic*)]
    \item
        We include an additional task \textsc{Mixture}, which samples from \textsc{EQ}, \textsc{Matern-$\frac52$}, \textsc{Noisy Mixture}, \textsc{Weakly Per.}, or \textsc{Sawtooth} with equal probability and thus consistutes a highly non-Gaussian mixture process.
    \item
        In line with the theoretical analysis in \cref{app:noisy_processes,app:objective},
        we contaminate all data samples with $\Normal(0,0.05^2)$-noise.
    \item
        The margin of the \textsc{ConvCNP} is reduced to $0.1$.
\end{enumerate}
The architecture for \textsc{GNP} is analogous to the architecture for \textsc{ConvCNP}, with the following exceptions:
\begin{enumerate}[label=(\arabic*)]
    \item
        To alleviate memory requirements, the receptive field size is limited to $8$.
    \item
        To alleviate memory requirements, the points per unit is decreased to $20$.
\end{enumerate}

Since the tasks are contaminated with noise, we extend the last step \circled{3} in the kernel the architecture (see \cref{subsec:kernel_architecture}) to also include a term for homogeneous noise, $\mK\us{(t)} = \operatorname{dec}(\mK, \vx\us{(t)}) + \sigma^2 \mI$ where $\sigma^2 > 0$ is a learnable parameter, which comes with the added benefit of stabilising the numerics during training.

\Cref{tab:1D_parameter_counts} show the parameter count for all models in all tasks in the 1D experiments.
The models were trained for roughly five days on a Tesla V100 GPU;
\cref{tab:1D_epoch_timings} shows the timings of a single epoch for all models in all tasks.
Note that, for the computationally more expensive tasks \textsc{Sawtooth} and \textsc{Mixture}, the timing of an epoch for the \textsc{GNP} increases drastically, which is likely due to excessive allocations on the GPU.

\Cref{tab:1D_results} shows (where applicable) the performance of the ground-truth Gaussian process (GP), the ground-truth Gaussian process without correlations (GP (diag.)), the \textsc{GNP}, the \textsc{ConvCNP} \citep{Gordon:2020:Convolutional_Conditional_Neural_Processes}, the \textsc{ConvNP} \citep{Foong:2020:Meta-Learning_Stationary_Stochastic_Process_Prediction}, and the \textsc{ANP} \citep{Kim:2019:Attentive_Neural_Processes} on all five data sets in (1) an interpolation setup, (2) an interpolation setup where the model was not trained, testing generalisation capability, and (3) and extrapolation setup, also testing generalisation capability.
Note that the translation equivariance built into the \textsc{ConvCNP}, \textsc{ConvNP}, and \textsc{GNP} enables these models to to maintain their performance when evaluated outside of the training range.

\clearpage

\newcolumntype{C}[1]{>{\centering\let\newline\\\arraybackslash\hspace{0pt}}m{#1}}
\begin{table}[h]
    \centering
    \small
    \begin{tabular}{rC{2.1cm}C{2.1cm}C{2.1cm}C{2.1cm}C{2.1cm}}
        \toprule
        & \scshape EQ & \scshape Mat\'ern--$\frac52$ & \scshape\footnotesize Weakly Per.\ & \scshape Sawtooth & \scshape Mixture \\ \midrule
        \scshape GNP & $101\,961$ & $101\,961$ & $147\,017 $ & $360\,009$ & $360\,009$ \\
        \scshape ConvCNP & $42\,822$ & $42\,822$ & $51\,014$ & $100\,166$ & $100\,166$ \\
        \scshape ConvNP & $88\,486$ & $88\,486$ & $104\,870$ & $104\,870$ & $203\,174$ \\
        \scshape ANP & $530\,178$ & $530\,178$ & $530\,178$ & $530\,178$ & $530\,178$ \\
        \bottomrule
    \end{tabular}
    \caption{Parameter counts for the \textsc{GNP}, \textsc{ConvCNP}, \textsc{ConvNP}, and \textsc{ANP} in the 1D regression tasks}
    \label{tab:1D_parameter_counts}
    \vspace*{-10pt}
\end{table}
\begin{table}[h]
    \small
    \centering
    \begin{tabular}{rC{2.1cm}C{2.1cm}C{2.1cm}C{2.1cm}C{2.1cm}}
        \toprule
        & \scshape EQ & \scshape Mat\'ern--$\frac52$ & \scshape\footnotesize Weakly Per.\ & \scshape Sawtooth & \scshape Mixture \\ \midrule
        \scshape GNP & $145$ & $145$ & $275$ & $1000$ & $1035$ \\
        \scshape ConvCNP & $30$ & $30$ & $40$ & $25$ & $60$ \\
        \scshape ConvNP & $85$ & $85$ & $110$ & $160$ & $190$ \\
        \scshape ANP & $40$ & $40$ & $50$ & $50$ & $75$ \\
        \bottomrule
    \end{tabular}
    \caption{Rough estimate of the numbers of seconds required for a single epoch of the \textsc{GNP}, \textsc{ConvCNP}, \textsc{ConvNP}, and \textsc{ANP} on a Tesla V100 GPU in the 1D regression tasks}
    \label{tab:1D_epoch_timings}
\end{table}
\begin{table}[h!]
\small
\centering
\begin{tabular}{lccccc}
\toprule
 & \multicolumn{1}{c}{\scshape EQ} & \multicolumn{1}{c}{\scshape Mat\'ern--$\frac52$} & \multicolumn{1}{c}{\scshape Weakly Per.} & \multicolumn{1}{c}{\scshape Sawtooth} & \multicolumn{1}{c}{\scshape Mixture}\\
\midrule\multicolumn{6}{l}{\textsc{Interpolation inside training range}} \\[0.5em]
\scshape GP & $0.70 { \scriptstyle \,\pm\, 4.8\text{\textsc{e}}{\,\text{--}3} }$ & $0.31 { \scriptstyle \,\pm\, 4.8\text{\textsc{e}}{\,\text{--}3} }$ & $\text{--}0.32 { \scriptstyle \,\pm\, 4.3\text{\textsc{e}}{\,\text{--}3} }$ & n/a & n/a \\
\scshape GP (\rm{diag.}) & $\text{--}0.81 { \scriptstyle \,\pm\, 0.01 }$ & $\text{--}0.93 { \scriptstyle \,\pm\, 0.01 }$ & $\text{--}1.18 { \scriptstyle \,\pm\, 7.0\text{\textsc{e}}{\,\text{--}3} }$ & n/a & n/a \\
\scshape GNP & $0.70 { \scriptstyle \,\pm\, 5.0\text{\textsc{e}}{\,\text{--}3} }$ & $0.30 { \scriptstyle \,\pm\, 5.0\text{\textsc{e}}{\,\text{--}3} }$ & $\text{--}0.47 { \scriptstyle \,\pm\, 5.0\text{\textsc{e}}{\,\text{--}3} }$ & $0.42 { \scriptstyle \,\pm\, 0.01 }$ & $0.10 { \scriptstyle \,\pm\, 0.02 }$ \\
\scshape ConvCNP & $\text{--}0.80 { \scriptstyle \,\pm\, 0.01 }$ & $\text{--}0.95 { \scriptstyle \,\pm\, 0.01 }$ & $\text{--}1.20 { \scriptstyle \,\pm\, 7.0\text{\textsc{e}}{\,\text{--}3} }$ & $0.55 { \scriptstyle \,\pm\, 0.02 }$ & $\text{--}0.93 { \scriptstyle \,\pm\, 0.02 }$ \\
\scshape ConvNP & $\text{--}0.46 { \scriptstyle \,\pm\, 0.01 }$ & $\text{--}0.67 { \scriptstyle \,\pm\, 9.0\text{\textsc{e}}{\,\text{--}3} }$ & $\text{--}1.02 { \scriptstyle \,\pm\, 6.0\text{\textsc{e}}{\,\text{--}3} }$ & $1.20 { \scriptstyle \,\pm\, 7.0\text{\textsc{e}}{\,\text{--}3} }$ & $\text{--}0.50 { \scriptstyle \,\pm\, 0.02 }$ \\
\scshape ANP & $\text{--}0.61 { \scriptstyle \,\pm\, 0.01 }$ & $\text{--}0.75 { \scriptstyle \,\pm\, 0.01 }$ & $\text{--}1.19 { \scriptstyle \,\pm\, 5.0\text{\textsc{e}}{\,\text{--}3} }$ & $0.34 { \scriptstyle \,\pm\, 7.0\text{\textsc{e}}{\,\text{--}3} }$ & $\text{--}0.69 { \scriptstyle \,\pm\, 0.02 }$ \\[.5em]
\midrule\multicolumn{6}{l}{\textsc{Interpolation beyond training range}} \\[0.5em]
\scshape GP & $0.70 { \scriptstyle \,\pm\, 4.8\text{\textsc{e}}{\,\text{--}3} }$ & $0.31 { \scriptstyle \,\pm\, 4.8\text{\textsc{e}}{\,\text{--}3} }$ & $\text{--}0.32 { \scriptstyle \,\pm\, 4.3\text{\textsc{e}}{\,\text{--}3} }$ & n/a & n/a \\
\scshape GP (\rm{diag.}) & $\text{--}0.81 { \scriptstyle \,\pm\, 0.01 }$ & $\text{--}0.93 { \scriptstyle \,\pm\, 0.01 }$ & $\text{--}1.18 { \scriptstyle \,\pm\, 7.0\text{\textsc{e}}{\,\text{--}3} }$ & n/a & n/a \\
\scshape GNP & $0.69 { \scriptstyle \,\pm\, 5.0\text{\textsc{e}}{\,\text{--}3} }$ & $0.30 { \scriptstyle \,\pm\, 5.0\text{\textsc{e}}{\,\text{--}3} }$ & $\text{--}0.47 { \scriptstyle \,\pm\, 5.0\text{\textsc{e}}{\,\text{--}3} }$ & $0.42 { \scriptstyle \,\pm\, 0.01 }$ & $0.10 { \scriptstyle \,\pm\, 0.02 }$ \\
\scshape ConvCNP & $\text{--}0.81 { \scriptstyle \,\pm\, 0.01 }$ & $\text{--}0.95 { \scriptstyle \,\pm\, 0.01 }$ & $\text{--}1.20 { \scriptstyle \,\pm\, 7.0\text{\textsc{e}}{\,\text{--}3} }$ & $0.53 { \scriptstyle \,\pm\, 0.02 }$ & $\text{--}0.96 { \scriptstyle \,\pm\, 0.02 }$ \\
\scshape ConvNP & $\text{--}0.46 { \scriptstyle \,\pm\, 0.01 }$ & $\text{--}0.67 { \scriptstyle \,\pm\, 9.0\text{\textsc{e}}{\,\text{--}3} }$ & $\text{--}1.02 { \scriptstyle \,\pm\, 6.0\text{\textsc{e}}{\,\text{--}3} }$ & $1.19 { \scriptstyle \,\pm\, 7.0\text{\textsc{e}}{\,\text{--}3} }$ & $\text{--}0.53 { \scriptstyle \,\pm\, 0.02 }$ \\
\scshape ANP & $\text{--}1.42 { \scriptstyle \,\pm\, 6.0\text{\textsc{e}}{\,\text{--}3} }$ & $\text{--}1.34 { \scriptstyle \,\pm\, 6.0\text{\textsc{e}}{\,\text{--}3} }$ & $\text{--}1.33 { \scriptstyle \,\pm\, 4.0\text{\textsc{e}}{\,\text{--}3} }$ & $\text{--}0.17 { \scriptstyle \,\pm\, 2.0\text{\textsc{e}}{\,\text{--}3} }$ & $\text{--}1.24 { \scriptstyle \,\pm\, 0.01 }$ \\[.5em]
\midrule\multicolumn{6}{l}{\textsc{Extrapolation beyond training range}} \\[0.5em]
\scshape GP & $0.44 { \scriptstyle \,\pm\, 2.9\text{\textsc{e}}{\,\text{--}3} }$ & $0.09 { \scriptstyle \,\pm\, 3.1\text{\textsc{e}}{\,\text{--}3} }$ & $\text{--}0.52 { \scriptstyle \,\pm\, 3.4\text{\textsc{e}}{\,\text{--}3} }$ & n/a & n/a \\
\scshape GP (\rm{diag.}) & $\text{--}1.40 { \scriptstyle \,\pm\, 6.7\text{\textsc{e}}{\,\text{--}3} }$ & $\text{--}1.41 { \scriptstyle \,\pm\, 6.6\text{\textsc{e}}{\,\text{--}3} }$ & $\text{--}1.41 { \scriptstyle \,\pm\, 5.6\text{\textsc{e}}{\,\text{--}3} }$ & n/a & n/a \\
\scshape GNP & $0.44 { \scriptstyle \,\pm\, 3.0\text{\textsc{e}}{\,\text{--}3} }$ & $0.08 { \scriptstyle \,\pm\, 3.0\text{\textsc{e}}{\,\text{--}3} }$ & $\text{--}0.62 { \scriptstyle \,\pm\, 4.0\text{\textsc{e}}{\,\text{--}3} }$ & $0.04 { \scriptstyle \,\pm\, 9.0\text{\textsc{e}}{\,\text{--}3} }$ & $\text{--}0.07 { \scriptstyle \,\pm\, 0.01 }$ \\
\scshape ConvCNP & $\text{--}1.41 { \scriptstyle \,\pm\, 7.0\text{\textsc{e}}{\,\text{--}3} }$ & $\text{--}1.42 { \scriptstyle \,\pm\, 6.0\text{\textsc{e}}{\,\text{--}3} }$ & $\text{--}1.41 { \scriptstyle \,\pm\, 6.0\text{\textsc{e}}{\,\text{--}3} }$ & $0.06 { \scriptstyle \,\pm\, 8.0\text{\textsc{e}}{\,\text{--}3} }$ & $\text{--}1.36 { \scriptstyle \,\pm\, 0.02 }$ \\
\scshape ConvNP & $\text{--}1.11 { \scriptstyle \,\pm\, 5.0\text{\textsc{e}}{\,\text{--}3} }$ & $\text{--}1.12 { \scriptstyle \,\pm\, 5.0\text{\textsc{e}}{\,\text{--}3} }$ & $\text{--}1.23 { \scriptstyle \,\pm\, 4.0\text{\textsc{e}}{\,\text{--}3} }$ & $0.88 { \scriptstyle \,\pm\, 9.0\text{\textsc{e}}{\,\text{--}3} }$ & $\text{--}0.93 { \scriptstyle \,\pm\, 0.01 }$ \\
\scshape ANP & $\text{--}1.31 { \scriptstyle \,\pm\, 5.0\text{\textsc{e}}{\,\text{--}3} }$ & $\text{--}1.28 { \scriptstyle \,\pm\, 5.0\text{\textsc{e}}{\,\text{--}3} }$ & $\text{--}1.32 { \scriptstyle \,\pm\, 5.0\text{\textsc{e}}{\,\text{--}3} }$ & $\text{--}0.17 { \scriptstyle \,\pm\, 1.0\text{\textsc{e}}{\,\text{--}3} }$ & $\text{--}1.11 { \scriptstyle \,\pm\, 0.01 }$ \\[.5em]
\bottomrule
\end{tabular}
\caption{
    Full results for 1D regression experiments.
    The numbers are average target point likelihood under the predictive distribution conditioned on the context set.
    The errors are 95\%-confidence intervals.
    See \citet{Foong:2020:Meta-Learning_Stationary_Stochastic_Process_Prediction} for more details.
}
\vspace*{-20pt}
\label{tab:full_1D_results}
\end{table}

\end{document}